\documentclass{article}

\usepackage[preprint, nonatbib]{neurips_2020}

\usepackage[utf8]{inputenc} 
\usepackage[T1]{fontenc}    
\usepackage{hyperref}       
\usepackage{url}            
\usepackage{booktabs}       
\usepackage{amsfonts}       

\usepackage{nicefrac}       
\usepackage{microtype}      
\usepackage{xcolor}

\usepackage{graphicx}
\graphicspath{ {./} }

\usepackage{amsthm}
\usepackage{amsmath}
\usepackage{caption}
\usepackage{makecell}

\newcommand{\eq}[1]{eq. \ref{eq:#1}}
\newcommand{\Eq}[1]{Eq. \ref{eq:#1}}

\newcommand{\sect}[1]{section \ref{sec:#1}}
\newcommand{\Sect}[1]{Section \ref{sec:#1}}
\newcommand{\sects}[2]{sections \ref{sec:#1} and \ref{sec:#2}}

\newcommand{\appendixSect}[1]{Supplement \ref{sec:#1}}

\newcommand{\Cra}[1]{Corollary \ref{corollary:#1}}

\newcommand{\Lma}[1]{Lemma \ref{lemma:#1}}

\newcommand{\Thm}[1]{Theorem \ref{theorem:#1}}

\newcommand{\Fig}[1]{Fig.~\ref{fig:#1}}

\newcommand{\Table}[1]{Table~\ref{tbl:#1}}

\newtheorem{theorem}{Theorem}
\newtheorem{lemma}{Lemma}
\newtheorem{corollary}{Corollary}

\newcommand{\lowerscript}[1]{\text{\usefont{U}{BOONDOX-cal}{m}{n}{#1}}}

\title{Deep probabilistic model synthesis enables unified modeling of whole-brain neural activity across individual subjects}

\author{%
  William E.~Bishop$^{1, *}$ \\
  \And
   Luuk W. Hesselink$^{2}$ \\
  \And
  Bernhard Englitz$^{2}$ \\
  \And
  Misha B. Ahrens$^{1 \dag}$ \\
  \And
  James E. Fitzgerald$^{1, 3, 4 \dag}$ \\
}

\begin{document}


\maketitle

$^1$Janelia Research Campus, Howard Hughes Medical Institute, Ashburn, VA;  $^2$Computational Neuroscience Lab,  Donders Center for Neuroscience, Radboud University, Nijmegen, The Netherlands; $^3$Departments of Neurobiology, Engineering Sciences and Applied Mathematics, and Physics and Astronomy, Northwestern University, Evanston, IL;  $^4$ NSF-Simons National Institute for Theory and Mathematics in Biology, Chicago, IL;  $^\dag$These authors contributed equally to this work; $^*$Now at Google Deepmind. To whom correspondence may be addressed: willbishop.neuro@gmail.com, ahrensm@janelia.hhmi.org, james.fitzgerald@northwestern.edu.

\begin{abstract}

Many disciplines need quantitative models that synthesize experimental data across multiple instances of the same general system. For example, neuroscientists must combine data from the brains of many individual animals to understand the species' brain in general. However, typical machine learning models treat one system instance at a time. Here we introduce a machine learning framework, deep probabilistic model synthesis (DPMS), that leverages system properties auxiliary to the model to combine data across system instances. DPMS specifically uses variational inference to learn a conditional prior distribution and instance-specific posterior distributions over model parameters that respectively tie together the system instances and capture their unique structure. DPMS can synthesize a wide variety of model classes, such as those for regression, classification, and dimensionality reduction, and we demonstrate its ability to improve upon single-instance models on synthetic data and whole-brain neural activity data from larval zebrafish.
\end{abstract}


\section{Introduction}

Researchers across many disciplines still struggle to understand complex systems \cite{newman2011complex}. A common challenge is that there is often no single canonical system to study, but instead many \emph{instances} of the same general type of system, each with its own variability. For example, biologists need data from many unique ecosystems to study organism-environment interactions, economists seek to predict price fluctuations across many distinct international markets, and physicists study many different materials to discover the general principles governing condensed matter systems. 

In this work, we consider the problem of fitting data-driven models of complex systems using measurements obtained from individual, unique instances of the same type of system. We refer to this problem as \emph{model synthesis}. We focus on scenarios where we seek to learn quantitative models from scratch, and where standard machine learning approaches are limited to fitting single system instances at a time. For example, the brains of distinct individuals each contain a unique set of neurons with varying trial-to-trial spontaneous activity and behavior, making direct correspondence across brains impossible \cite{Stringer2018, mu2019glia}. Standard single-neuron-resolution decoding models \cite{degenhart2020}, neural network models \cite{pillow2008}, functional clustering models \cite{chen2025whole}, and dimensionality-reduction models \cite{stringer2019high} must therefore be fit to one system instance at a time. This limits both generalizability, because we are describing single system instances and not the general system, and model detail, because each system provides only limited data.

Our approach to model synthesis predicts model parameters for each system instance from system properties that are auxiliary to the model. For example, a neuron's role in generating behavior varies depending on its cell type and brain region  \cite{kato2015global, ohyama2015multilevel, naumann2016whole}, and the weight parameters of decoding models might be similarly predictable from the genetic identities, functional fingerprints, and/or locations of neurons in the brain. Since available system properties can typically only predict model parameters imperfectly, we specifically focus on the more general problem of \emph{probabilistic model synthesis}. Probabilistic model synthesis learns a conditional prior distribution (CPD) that ties together the system instances, as well as posterior distributions that account for the observed data and unique structure of each instance. 

We implement probabilistic model synthesis using general and flexible methods from variational inference \cite{blei2017}. We refer to our resulting implementation as \emph{deep probabilistic model synthesis} (DPMS) as the models, priors, and posteriors utilize deep neural networks. We show theoretically that DPMS can identify common structures linking properties and model parameters, while preserving individual variability, and we demonstrate DPMS' use in practice by applying it to both synthetic and real data. Our real-world applications focus on neuroscience, where we demonstrate that it improves whole-brain regression and dimensionality reduction models for larval zebrafish behaving in virtual reality environments \cite{ahrens2013, vladimirov2014, chen2018}. Neuroscience is a natural use case for DPMS as understanding the brain requires the synthesis of results from more experiments than can conceivably be performed in any individual animal \cite{biswas2020}. 

\section{Results}

\subsection{Theoretical framework}
\label{sec:DPMS}

We introduce probabilistic model synthesis in the context of input-output models such as classification and regression models (see Methods for more details and other model types, such as those for dimensionality reduction). We aim to fit parameterized models to each of $S$ system instances indexed by $s \in {1, \ldots, S}$. For each system instance $s$, we denote the probability of the observed input-output data under the parameterized model as $p(Y^s| X^s, \theta^s)$, where the dimensionalities of the data $(X^s, Y^s)$ and model parameters $\theta^s$ may vary across system instances (Methods \ref{sec:cpd}). In a neuroscience setting, $\theta^s$ could describe how the activities of different neurons $X^s$ combine to drive continuous or categorical descriptions of behavior $Y^s$ (\Fig{concept}a). Probabilistic model synthesis predicts the model parameters from auxiliary properties of the system, $M^s$, by learning a \emph{conditional prior distribution} (CPD), $p_\gamma(\theta^s | M^s)$, that encodes beliefs about model parameters prior to observing $(X^s, Y^s)$ (\Fig{concept}a). This prior, parameterized by $\gamma$, therefore models the relationship between properties, $M^s$, and model parameters, $\theta^s$. In the neuroscience example, $M^s$ could encode each neuron's spatial coordinates or genetic profile, since these properties partially predict a neuron's role in neural processing and behavior, and $\gamma$ could be a high-dimensional vector designed to parametrizes a flexible class of probability distributions.

One could in principle learn the CPD by finding the $\gamma$ that maximizes the conditional likelihood of the observed data for all system instances, 
\begin{align}
 p_\gamma(\{Y^s\}_{s=1}^S|\{X^s, M^s\}_{s=1}^S) =  \prod_{s=1}^S \int p(Y^s|X^s, \theta^s)p_\gamma(\theta^s|M^s)d\theta^s.
\label{eq:cond_ll}
\end{align}
Applying Bayes rule then yields data-dependent posteriors for each system instance,
\begin{align}
p_\gamma(\theta^s| Y^s, X^s, M^s) = \frac{p(Y^s|X^s, \theta ^s)p_\gamma(\theta^s | M^s)}{\int p(Y^s|X^s, \theta ^s)p_\gamma(\theta^s | M^s) d\theta^s}.
\label{eq:exact_posterior}
\end{align}
However, evaluating these integrals is typically computationally infeasible. 

\begin{figure}[ht!]
\centering
\includegraphics[width=5.5in]{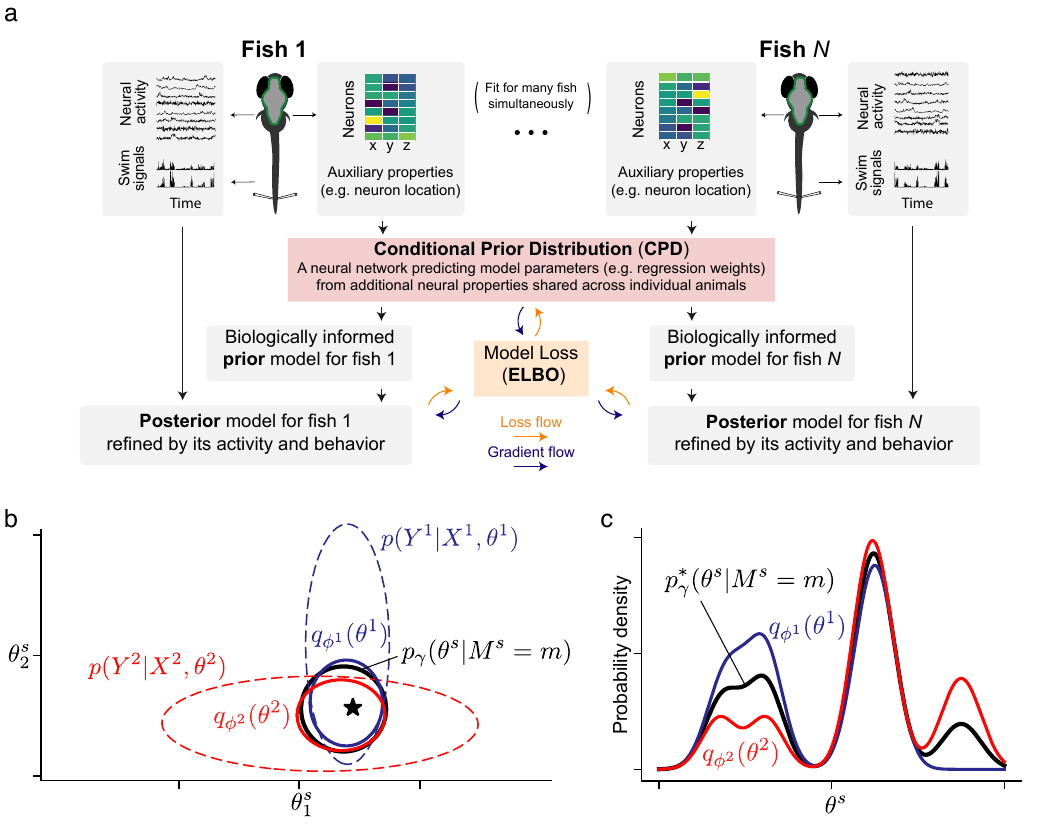}
\caption{\emph{Probabilistic Model Synthesis enables structure learned from one system instance to be transferred to another.} (\emph{\textbf{a}}) DPMS synthesizes models across system instances (e.g., individual animals) by predicting model parameters for each instance from auxiliary properties (e.g., neuron positions) through a shared Conditional Prior Distribution (CPD). The CPD, generally implemented as a deep neural network, models a distribution over model parameters conditioned on these auxiliary properties. This biologically-informed prior for each system instance is then refined using data (e.g., neural activity and behavior) collected from that instance, yielding system-specific posteriors. Both the CPD and the model posteriors are optimized using the data from all individual animals through the Evidence Lower-Bound (ELBO). This enables the CPD to capture a \textit{common} mapping from a neuron’s properties to its role in shaping dynamics and behavior, while the posteriors account for \textit{individual} variability. (\emph{\textbf{b}}) Illustration of probabilistic model synthesis in a scenario where two system instances (blue and red) have the same measurable properties and true model parameters (black star). All probability distributions are assumed to be multivariate normal, with level sets plotted as ellipses. We suppose that data are collected from the two system instances under different conditions, such that the data only constrain one of the two model parameters well for each system (blue and red dashed lines). Learning synthesizes information so that both parameters become well constrained in the CPD (black solid line) and optimal approximate posteriors (blue and red solid lines). Note that each approximate posterior is tighter than the CPD along the dimension well constrained by the data. (\emph{\textbf{c}}) As described in the text, the optimal CPD (black) is the average of approximate posteriors (red and blue) for system instances with the same measurable properties. Here, $\theta^s$ is one-dimensional, and we show the approximate posteriors (red and blue) for two system instances with measurable properties equal to $m$, and the optimal CPD conditioned on $m$ (black). 
}
\label{fig:concept}
\end{figure}

Deep probabilistic model synthesis (DPMS) provides a general and tractable approach using variational inference (\Fig{concept}a) \cite{jordan1999, blei2017}. DPMS specifically introduces approximate posteriors, $q_{\phi^s}(\theta^s)$, parameterized by $\phi^s$ for each system instance, and jointly optimizes $\{\phi^s\}_{s=1}^S$ and $\gamma$ to maximize 
\begin{align}
        \mathcal{L}(\{\phi^s\}_{s=1}^S, \gamma) = \sum_{s=1}^S \mathbb{E}_{q_{\phi^s}(\theta^s)}\left[ \log p(Y^s|X^s, \theta^s) \right] - \text{KL}\left[q_{\phi^s}(\theta^s) || p_\gamma(\theta^s|M^s) \right]\label{eq:ELBO}
\end{align}
which maximizes the log probability of the data and minimizes the mean Kullback-Leibler (KL) divergence from $q_{\phi^s}(\theta^s)$ to $p_\gamma(\theta^s| Y^s, X^s, M^s)$ (Supplement \ref{sec:ELBO_all}). $\mathcal{L}$ is usually termed the evidence lower bound (ELBO) as it bounds the log-likelihood of the observed data from all system instances as
\begin{align}
        \log p_\gamma(\{Y^s\}_{s=1}^S|\{X^s, M^s\}_{s=1}^S) \geq \mathcal{L}(\{\phi^s\}_{s=1}^S, \gamma). 
\end{align}
The form of each $q_{\phi^s}$ can be very flexible, taking advantage of modern advances in probabilistic modeling (\cite{papamakarios2021}) and automatic differentiation methods (\cite{abadi2016, paszke2019}). In practice, we follow previous work and approximate the expected log-likelihood and KL divergence by sampling (Methods \ref{sec:sampling_mds}, \cite{williams1992, kingma2014, roeder2017, mnih2014}). A closely related formalism also applies to dimensionality reduction models (Methods \ref{sec:dim_red_mdls}). 

The structure of the ELBO explains how DPMS uses the CPD to synthesize models across system instances. Maximizing \Eq{ELBO} encourages posteriors to reflect a similar structure across system instances, since the CPD must remain close to each. The CPD effectively pools the data from different system instances enabling model parameters from individual systems to benefit from this learned shared structure (\Fig{concept}b). At the same time, the CPD can reflect uncertainty about model parameters arising from variability not predicted by measurable properties (Supplement \ref{sec:proofs}). In particular, the optimal CPD averages posteriors over model parameters for system instances with the identical measurable properties (\Fig{concept}c). 

Selecting the form of the CPD requires a balance between its flexibility and the ability to achieve synthesis. For example, when $M^s$ is unique for each system instance, the term $\text{KL}\left[q(\theta^s; \phi_\theta^s) || p(\theta^s|M^s; \gamma) \right]$ in \Eq{ELBO} can be trivially minimized by learning a CPD where $p(\theta^s|M^s; \gamma) = q(\theta^s; \phi_\theta^s)$ for each $s$. In this scenario, the CPD effectively memorizes a unique distribution over parameters for each system instance, failing to yield synthesis. This can be prevented by constraining the CPD or the feature space for properties in a way that limits its ability to become overly specialized for small regions of property space, for instance through discretization (Methods \ref{sec:SHBF_fcns}). 

Uncertainty estimates of the CPD, such as variance, must also be interpreted with care, as the CPD will learn point estimates of the model parameters unless there is variability in the data that cannot be predicted by auxiliary properties. This means that in data-limited regimes, the CPD may learn to predict parameters with very low variance. Indeed, when there is not enough data to detect or model variability in a system's parameters (e.g. if one only has one instance), a reasonable approach is to learn a single deterministic relationship. When additional data is observed, variability can be recognized and will be represented as variability over model parameters in the CPD. We explore this behavior in detail for a simple but extreme example of synthesizing linear models when only one sample is observed from each system instance in the supplement in \Fig{low_data_synth}.

In practice, it may be unnecessary or undesirable to predict all model parameters from system properties, and there may be certain parameters we desire to rigidly share across models. DPMS can accommodate this by splitting the parameters into three sets (see Methods \ref{sec:practical_concerns}). The first set of parameters, $\theta^s_\text{props}$, is the previously emphasized set of parameters that we seek to predict from the system properties. The second set of parameters, $\theta_\text{shared}$, is a set of \emph{core} parameters that are shared across system instances. For simplicity, we choose to learn point estimates for $\theta_\text{shared}$ in this work. The final set of parameters, $\theta^s_\text{no-props}$, are those that can vary across system instances but that we do not seek to predict from the system's properties. 
We learn non-conditional priors and form approximate posteriors over these parameters for each example system. In this work, we use the same priors across system instances, requiring the dimensionality of $\theta^s_\text{no-props}$ to be identical across system instances. 

\subsection{Illustrating DPMS through a synthetic example}
\label{sec:syn_example}

We first illustrate key features of DPMS using a synthetic scenario where the ground-truth model structure is known. We generated system instances matching the model structure later used to describe brain-wide neural influences on behavior and refer to each instance as the simulated brain of an individual. Conceptually, each simulated brain is composed of two components, a linear projection to a shared low-dimensional space and a subsequent non-linear mapping to behavior (\Fig{syn_results}a). The linear projection varies across individuals, whereas the non-linear mapping is shared. In the context of neuroscience, these two components might represent how individual brains with low-level differences, such as the number of neurons and their connectivity, can implement the same high-level algorithm for driving behavior \cite{churchland2012, mante2013, gallego2017}. We generated 100 simulated brains, sampling the number of neurons $d_s^x$ uniformly in $[10^4, 1.1\times10^4]$ and their positions uniformly in the unit square. Projection weights $\omega^s[i]$ were drawn from $\mathcal{N}(\mu(M^s[i,:]), \sigma^2(M^s[i,:]))$ for randomly generated functions $\mu$ and $\sigma$, where $\omega^s[i]$ is the projection weight of neuron $i$ in system $s$ and $M^s[i,:]$ is its 3D position (\Fig{syn_results}b). This generative process defines the ground-truth CPD. See Methods \ref{sec:syn_mdl_fitting} and Supplement \ref{sec:simulated_ex_details} for full simulation details.

\begin{figure}[h]
\centering
\includegraphics[width=5.0in]{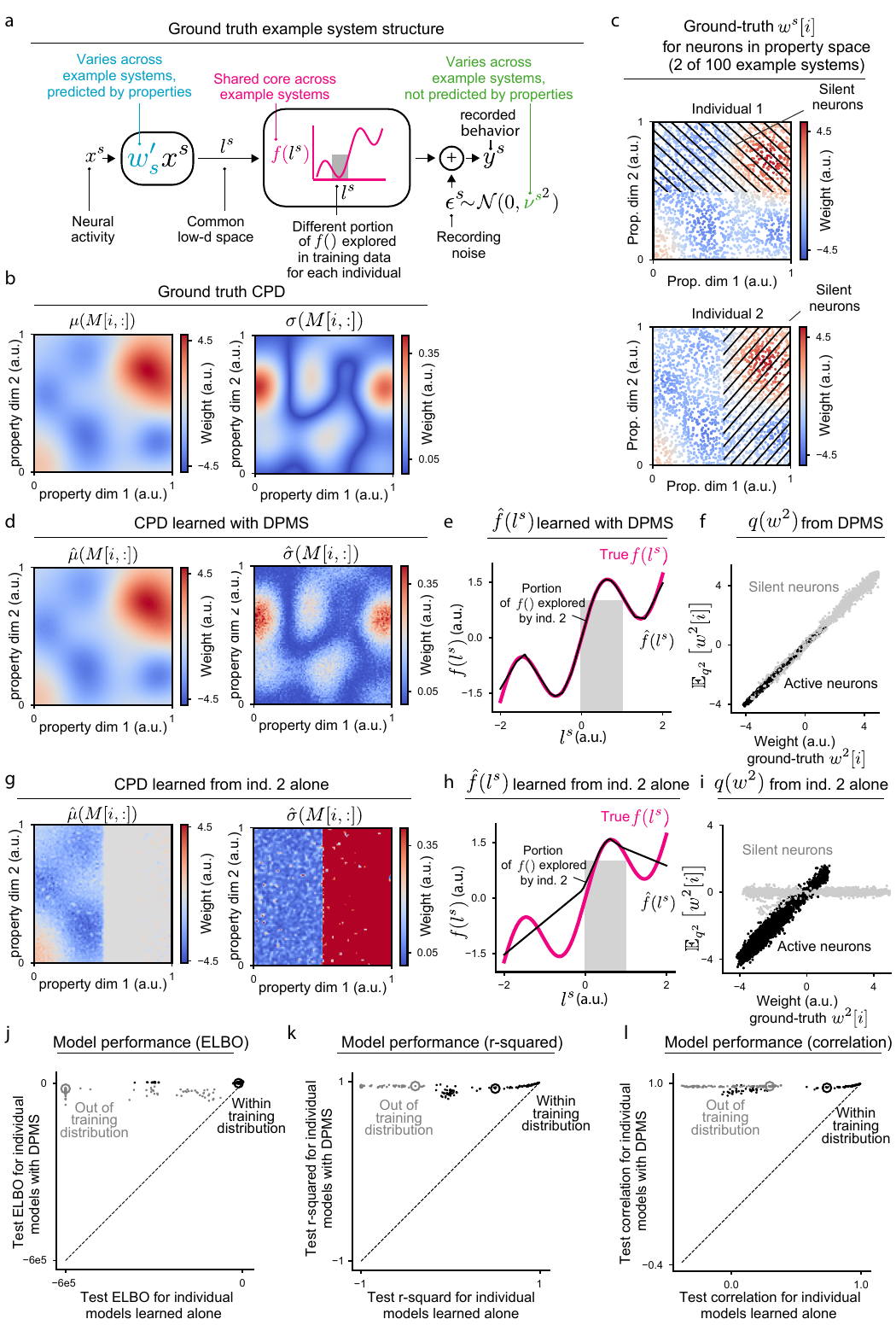}
\caption{Caption on next page.}
\label{fig:syn_results}
\end{figure}

\begin{figure}[h]
\ContinuedFloat
\caption{\emph{DPMS applied to a simulated scenario.} (\emph{\textbf{a}}) The structure of the ground-truth simulated brain models generated for each individual. The neuronal activity, $x^s$, is projected into a 1D subspace to form $l^s$ representing the conserved computational quantity the model brain uses to drive behavior. Projection weights vary across neurons and system instances in a way that depends on properties (see panel \emph{b}). A function, $f$, that is shared across individuals and represents the high-level algorithm the brain uses to drive behavior transforms $l^s$ into behavior. Recorded behavior, $y^s$, is formed by adding recording noise to $f(l^s)$ with a standard deviation $\nu^s$, selected independently from a Gamma distribution across individuals. Only a pseudo-randomly selected portion of $f$ is explored in the training data for each individual. (\emph{\textbf{b}}) The ground-truth functions, $\mu$ and $\sigma$, specifying the mean and standard deviation of the ground-truth CPD throughout the 2D property space. (\emph{\textbf{c}}) Ground-truth weights for neurons visualized in property space for two system instances. Neurons in a pseudo-randomly selected half of property space (dashed regions) are silent in the training data for each individual. For visualization, only $10\%$ of the neurons have been been shown for each individual. (\emph{\textbf{d}}) The functions $\hat{\mu}$ and $\hat{\sigma}$ learned for the CPD by DPMS. (\emph{\textbf{e}}) The true shared function, $f$, and its estimate from DPMS, $\hat{f}$, shown over the entirety of their domain. The portion of $l^s$ explored for individual 2 is denoted in gray. (\emph{\textbf{f}}) The posterior mean over weights for active and silent neurons estimated by DPMS for individual 2. Note that $w^2$ in the figure denotes weights for individual 2 and not squared weights. (\emph{\textbf{g-i}}) Same as panels  \emph{d}-\emph{f} but fit to individual 2 in isolation. For visualization purposes, estimated standard deviation values have been clipped at $0.40$ (\emph{\textbf{j}}) Performance, quantified by ELBO, for all 100 simulated individuals when models were synthesized with DPMS or fit to each individual in isolation. Points show individual systems; large circles denote individual 2.  Across individual systems, the mean ± standard error of the ELBO for DPMS (individual fits) was 3720 ± 243 (-35,730 ± 8,984) when evaluated within domain and -29,120 ± 1,139 (-424,400 ± 21,080) when evaluated out of domain. (\emph{\textbf{k,l}}) Same as panel \emph{j}, with performance quantified by R-squared or by the correlation between predicted and recorded $y$ values. Mean R-squared across individuals was 0.929 ± 0.005 (0.480 ± 0.043) in domain and 0.950 ± 0.001 (-0.859 ± 0.021) out of domain.  Mean correlation was 0.965 ± 0.002 (0.585 ± 0.041) in domain and 0.975 ± 0.0004 (0.047 ± 0.025) out of domain. R-squared values were clipped at $-1$. Out-of-distribution R-squared and correlation can exceed in-distribution values because the out-of-distribution evaluation spans the full range of $f$, whereas each individual's in-distribution data covers only a limited portion.}
\end{figure}

To demonstrate DPMS' power, we generated training data with three limitations. First, the number of time points we sampled for each individual, $n^s$, was drawn uniformly from $[7500, 9000]$, yielding less samples for model fitting than the number of model parameters. Second, we simulated variable recording conditions, where neurons in half of the brain were inactive for each individual (\Fig{syn_results}c) making half of the projection weights $\omega^s$ unidentifiable from any single individual. Third, we assumed that variable recording conditions also drove a limited range of behavior. In particular, although the domain of $f$ was $(-2, 2)$, for each individual we generated activity for the non-silent neurons to ensure that $x^s_t$ only produced $l^s_t$ within an interval of length 1 (\Fig{syn_results}a, gray region in plot of $f$), preventing the data from any individual representing the full shape of $f$.

Despite these challenges, DPMS accurately recovered the ground truth model structure (\Fig{syn_results}d,e). The learned approximate posteriors accurately predicted the weights for all neurons (\Fig{syn_results}f), including those silent in a given individual's data. As expected, single-individual data constrained estimates only within the behavioral and activity ranges that specific individual exhibited (\Fig{syn_results}g-i). DPMS' synthetic properties can thus enable predictions that would otherwise be impossible. 

Finally, we evaluated the synthesized models on test data. Using posterior means to estimate parameters, we assessed performance on test datasets both within and outside each individual’s training distribution (Methods \ref{sec:syn_mdl_fitting}). Models fit with DPMS performed well both within and outside of the training distribution (\Fig{syn_results}j-l). In contrast, models estimated from single individuals performed poorly on out-of-distribution data and some even on within-distribution test-data.

\subsection{Synthesizing regression models for decoding behavior from neural population activity}
\label{sec:regResults}

In the next two sections, we demonstrate the utility of DPMS for neuroscience by using it in two different applications. Both applications use previously reported recordings of whole-brain neural activity in larval zebrafish responding to a variety of visual stimuli (Methods \ref{sec:data_preprocessing}, \cite{chen2018}).  Approximately $80,000$ individual neurons densely covering the entire brain of each fish were recorded with calcium imaging while behavior was simultaneously recorded in the form of the activity of axial motor neurons on the left and right sides of the tail (\Fig{reg_results}a,b). From the original recordings, we focused on eight animals imaged at similar frame rates displaying robust behavioral signals. We used the 3D positions of individual neurons, registered to a standard anatomical atlas \cite{randlett2015} as the properties enabling DPMS across fish. See Methods \ref{sec:app_details_reg} and Supplement \ref{sec:reg_real_data_details} for full application details. 

\begin{figure}[h]
\centering
\includegraphics[width=5.5in]{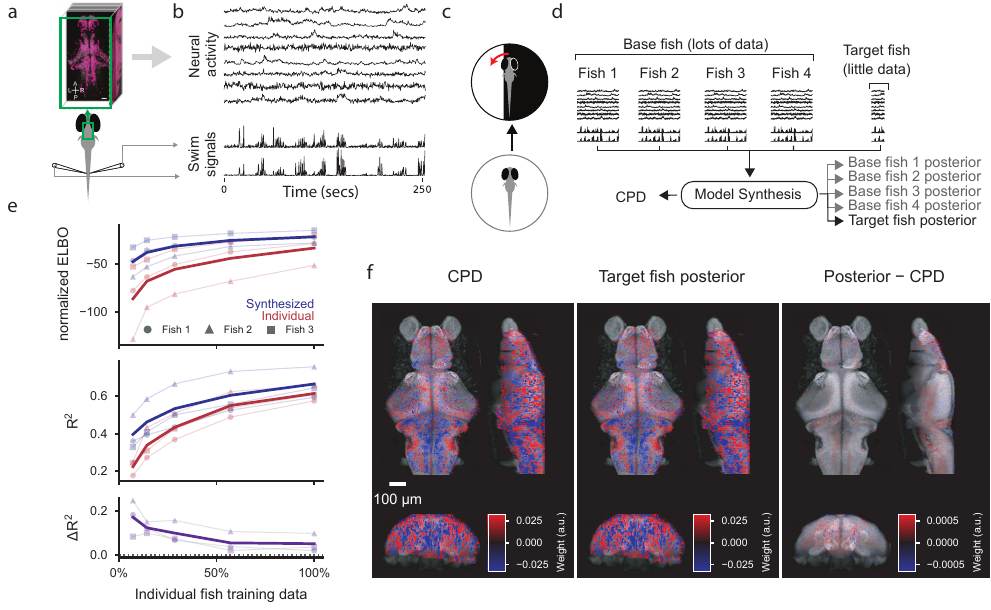}
\caption{\emph{DPMS applied to neural recordings in zebrafish larvae when synthesizing regression models  uncovers common structure across individuals.} (\emph{\textbf{a}}) The brain-wide  activity of ${\sim}80,000$ neurons per fish in a virtual reality setup was recorded along with the voltage of motor nerves on both sides of a fish's tail, which served as fictive swim signals. (\emph{\textbf{b}}) Example neural activity and fictive swim signals for one fish. (\emph{\textbf{c}}) We study the fish under phototaxis, in which they turn towards the brighter half of an arena, which alternated throughout an experiment. (\emph{\textbf{d}}) The data used for and the results of DPMS. Each base fish has more data than the target fish.  Model synthesis produces both a CPD and posteriors for each fish.  We focus on posteriors for the target fish. (\emph{\textbf{e}}, top) Cross-validated model performance as measured by the normalized ELBO, approximated from the test data. Models for the target fish were synthesized with the base fish (blue lines) or fit to the target fish alone (red lines). Results showing the average across folds for each individual target fish are shown in the light lines with same colors. Averages across fish are shown in the thicker lines. (\emph{\textbf{e}}, middle) Cross-validated prediction performance measured by the $R^2$ between the predicted and true behavioral traces, and (\emph{\textbf{e}}, bottom) the difference, $\Delta R^2$, indicating the performance gain in synthesizing models using DPMS. (\emph{\textbf{f}}) Max projections of the means of an example CPD, approximate posterior, and the difference between the two for the weights of neurons projecting to one of the dimensions of the low-d space. Weights are shown for target fish 1 when $100\%$ of the available training data in a fold was used.}
\label{fig:reg_results}
\end{figure}

We first applied DPMS to fit decoding models that predict behavior (motor activity driving swimming) from neural activity evoked by visual phototaxis stimuli (\Fig{reg_results}c). We imagine a scenario in which we are able to record a large amount of data from a \emph{base} set of fish but only limited data from a \emph{target} fish. For each target fish, we synthesized models, learned in conjunction with data from the base fish (\Fig{reg_results}d), and compared the performance of these models synthesized from data of the target fish alone.

Models for target fish synthesized with the base fish outperformed models fit to data from the target fish alone. We used the testing data for each target fish to quantify the predictive performance of each model with the normalized ELBO and the $R^2$ between the recorded swim signals and the swim signals predicted by the posterior means (\Fig{reg_results}e, top and middle). We further quantified the difference in prediction performance, $\Delta R^2$, between the synthesized model compared to the individual model (\Fig{reg_results}e, bottom). This difference indicated consistent improvement for each target fish through the application of DPMS.


Comparing CPD means and approximate posteriors learned for one target fish illustrates the results of synthesis (\Fig{reg_results}f). Weights in the posterior part of the brain indicate lateralized signals, perhaps important for predicting turning behavior \cite{chen2018}. While there are slight differences between the CPD and posterior, they correspond to a very large degree, indicating that model synthesis found a solution for the particular target fish that was in accord with the generic structure learned across fish. 

\subsection{Synthesizing dimensionality reduction models to find a shared latent space across individuals and experimental conditions}\label{sec:dimRedResults}

\begin{figure}[h]
\centering
\includegraphics[width=5.5in]{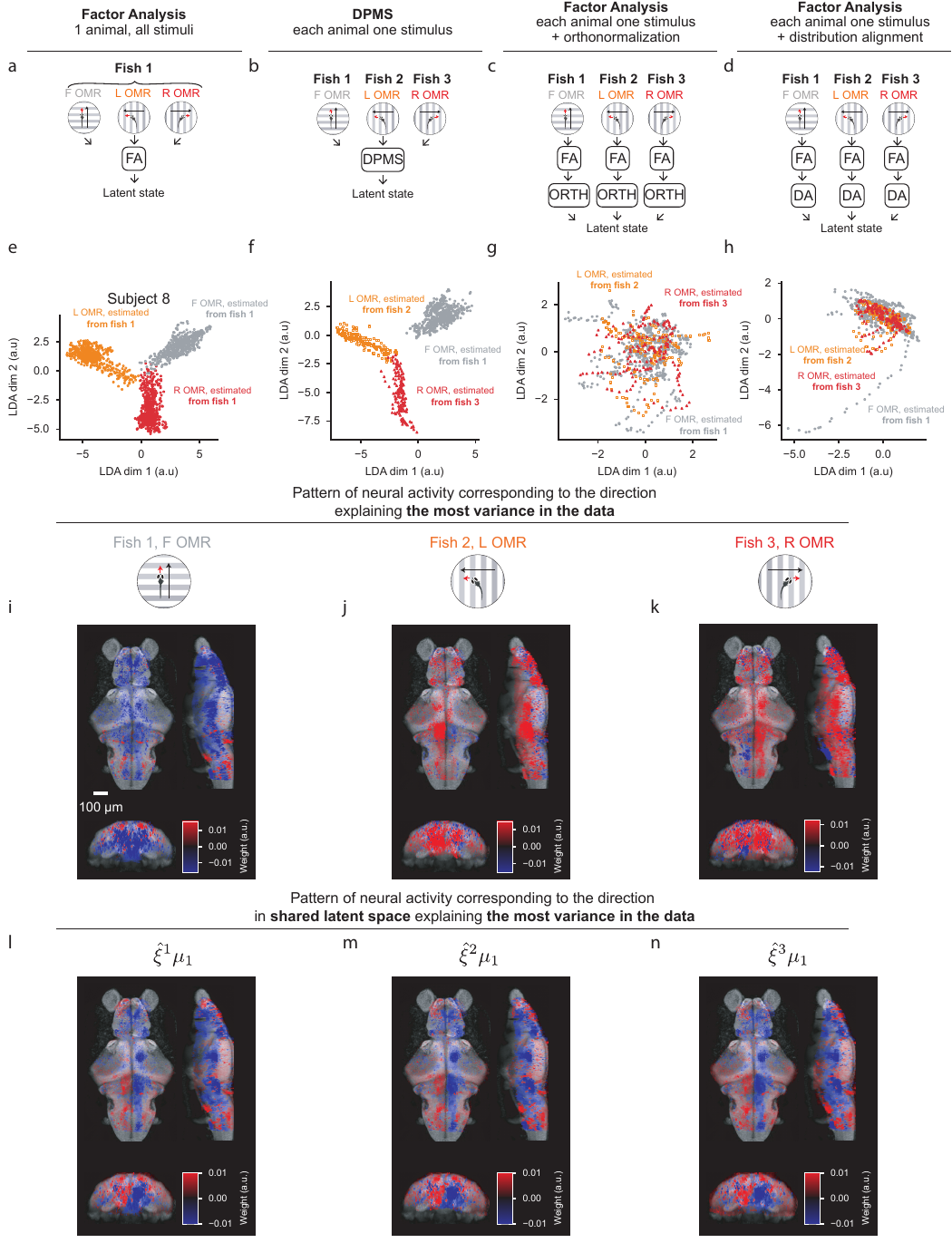}
\caption{Caption on next page.}
\label{fig:dr_lda}
\end{figure}

\begin{figure}[h]
\ContinuedFloat
\caption{\emph{DPMS synthesizing dimensionality reduction models can relate different behaviors observed in different fish in the same low-dimensional space.} (\emph{\textbf{a}}) We estimated latent state for each fish using data from all behaviors. (\emph{\textbf{b}}) DPMS was also applied to estimate latent state when only a single behavior was observed in each fish. (\emph{\textbf{c,d}}) To compare to existing methods, we applied standard factor analysis to the same data used in synthesis, fitting individual models to the data for each fish and then used orthonormalization, (c), or distribution alignment (DA, d) to attempt to put the latent state estimates across fish in the same space. (\emph{\textbf{e}}) Latent state estimated for example fish 1 when standard factor analysis was applied to data recorded under all three behaviors. Latent state is is shown in the best two-dimensional space for differentiating behavior.  (\emph{\textbf{f}}-\emph{\textbf{h}}) From left to right, latent state estimated with DPMS, applying orthonormalization to latent state estimated across fish with standard factor analysis, and applying DA to latent state estimated across fish with standard factor analysis, each when data from only a single behavior is observed in each fish. Latent state is again shown in the best two-dimensional spaces for differentiating behavior for each approach, and coordinate axes have been reflected and rotated to visually correspond to those in panel \emph{e}. (\emph{\textbf{i}}-\emph{\textbf{k}}) Patterns of neural activity explaining the most variance in the neural data observed for each fish. The most prominent pattern in the data for each fish is strongly influenced by the behavior of the fish. Shown are projections from the top, the side, and the front of the brain; longitudinal size is $\sim$ 800 $\mu$m. (\emph{\textbf{l}}-\emph{\textbf{n}}) Patterns of neural activity corresponding to the direction, $u_1$, in the shared latent space explaining the most variance across fish for each fish. The behavioral condition that fitting data for each fish was collected under is indicated at the top of each panel. The patterns identified across fish are very similar, illustrating how DPMS synthesizes models with similar mappings from latent state to observed neural activity across fish.}
\end{figure}

Our second neuroscience application used DPMS to fit dimensionality reduction models that permit high-dimensional neural activity to be represented in a low-dimensional space (Methods \ref{sec:dim_red_mdls}). Different behavioral conditions elicit distinct low-dimensional activity patterns, and capturing the diversity of all possible animal behaviors in a low-dimensional space requires the synthesis of data from more behavioral paradigms than any one animal could experience. Here we test whether DPMS can synthesize dimensionality reduction models that correctly relate neural activity for different behaviors in a shared low-dimensional space and predict the structure of brain-wide activity for all behaviors even though each fish exhibits only one behavior, a key challenge for existing methods \cite{dabagia2022}. See Methods \ref{sec:dr_app_details} and Supplement \ref{sec:add_fa_methods} for full application details. 

We fit factor analysis models to the brain-wide activity of fish performing the optomotor response (OMR), a behavior in which fish swim in the direction of moving gratings projected below them \cite{naumann2016whole,chen2018}. The fish responded to visual stimuli (whole-field gratings) moving forward, left, or right, producing three classes of behavior (typically, swims and turns in the direction of visual motion) (\Fig{dr_lda}a). However, we use data from only a single and distinct behavior for each of the three fish for model fitting(\Fig{dr_lda}b-d). 

We first asked whether models synthesized by DPMS correctly related the latent state associated with different behaviors across fish. To visualize latent structure, we applied linear discriminant analysis (LDA) to the 10-dimensional latent state variables to obtain a two-dimensional space that best separated behaviors. In the latent spaces estimated from each fish’s own data, we observed consistent clustering by behavior (fish 1 latents shown in \Fig{dr_lda}e). Furthermore, latent variables during periods without swimming lay near the center of the space, whereas those during swimming lay farther from the center (\Fig{low_data_synth_supp}a). This same structure was preserved when we applied DPMS to uncover LDA latents with each fish observing a single distinct behavior (\Fig{dr_lda}f, \Fig{low_data_synth_supp}b). Similar correspondence appeared when visualizing the three latent dimensions capturing most variance (\Fig{low_data_synth_supp}c,d). Notably, axis orthonormalization or distribution alignment (Supplement \ref{sec:add_fa_methods}, \cite{Courty2017, flamary2021pot}) failed to align latent spaces, underlining the non-triviality of these results (\Fig{dr_lda}g-h, \Fig{low_data_synth_supp}e-f).

To understand why DPMS correctly related latent states across animals, we examined the synthesized mappings from latent space to neural activity for each fish. Because each fish expresses variance dominated by different behaviors (\Fig{dr_lda}i–k), similarity across synthesized mappings would be highly non-trivial. Nevertheless, when identifying the latent-space direction that explained the most variance across all fish (Supplement \ref{sec:add_fa_methods}), we found that the synthesized FA models exhibited strikingly similar structure (\Fig{dr_lda}l–n). This dimension accounted for 28.7\% of total variance.

\begin{figure}[h]
\centering
\includegraphics[width=5.5in]
{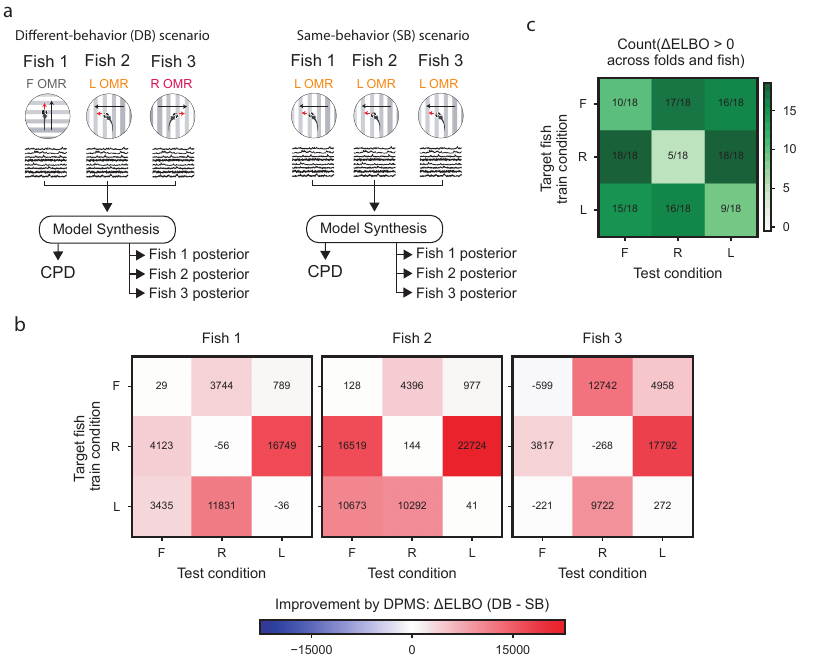}
\caption{\emph{Quantification of DPMS applied
to dimensionality reduction models.}
(\emph{\textbf{a}}) Factor-analysis models were
synthesized in two data quantity matched scenarios. In
the different-behavior (DB) scenario, non-
target fish contributed neural activity
recorded during stimulus conditions different
from the target-fish training condition (left).
In the same-behavior (SB) scenario, non-target
fish contributed data recorded during the same
stimulus condition as the target fish (right).
In both cases, DPMS learns a single CPD over
factor-analysis model parameters, together with
approximate posteriors over parameters and
latent states for each fish.
(\emph{\textbf{b}}) Improvement in held-out
ELBO for DB relative to SB models, shown
separately for each target fish after averaging
across 6 target-fish test folds. Rows
indicate the target-fish training condition and
columns indicate the test condition. 
(\emph{\textbf{c}}) Number of fish-fold comparisons in which DB models outperformed SB models for each train-test condition pair, measured by positive held-out ELBO improvement. Each entry comprises 18 comparisons (3 target fish $\times$ 6 folds). 
}
\label{fig:dr_results}
\end{figure}

We next quantified how well DPMS can synthesize
models that account for diverse behaviors
across animals. In each experiment, we
designated a \emph{target fish} and synthesized
models using data from the target fish and two
additional fish. To verify that any improvement
reflects exposure to a wider set of behaviors
rather than access to more data, we compared
two scenarios matched for the amount of
training and validation data. In the
\emph{different-behavior} (DB) scenario, the
two non-target fish provided data from
behaviors different from those of the target
fish (\Fig{dr_results}a, left). In the
\emph{same-behavior} (SB) scenario, the non-
target fish provided data from the same
behavior as the target fish (\Fig{dr_results}a,
right, see \sect{add_fa_methods} for details).

Models synthesized under the DB scenario
outperformed those synthesized under the SB
scenario when tested on behaviors
different from those used for training. This
effect was consistent across target animals
(\Fig{dr_results}b) and also when pooling results across fish
(\Fig{dr_results}c). Across fish test folds, DB models
accounted for neural activity better than SB models in 100 out of 108  folds (the off diagonal entries of \Fig{dr_results}c) when the test behavior was different than the behavior represented in the target's fish training data (exact two-sided sign-count test, Holm-corrected, p=$4.71\times10^{-21}$). Improvements when the test behavior matched that already in the training data for a target fish (the diagonal entries of \Fig{dr_results}c) were less consistent, occurring in 24 out of 54 folds, and were not significant (exact two-sided sign-count test, Holm-corrected, $p=0.50$). This suggests that the data for the target fish was abundant enough to enable fitting accurate models accounting for the aspects of neural activity present during the behavior represented in their training data.

Together, these results show that exposure to behavioral diversity across animals improves synthesis of target-fish models to held-out behaviors, and provide, to our knowledge, the first demonstration of a method that places representations of brain-wide neural activity from different individuals performing completely different behaviors into a shared low-dimensional space.

\section{Discussion}

In this work we propose a new machine learning problem - model synthesis - and a tractable approach to a probabilistic version of this problem, deep probabilistic model synthesis (DPMS). DPMS leverages variational inference to provide a flexible approach to fitting a large range of models. We analyzed this variational approach and showed how it performs synthesis by discovering common relationships between system properties and model parameters, while remaining capable of respecting variability between system instances. We further demonstrated the utility of DPMS on simulated and real examples by synthesizing models that were more accurate and complete than those learned from single individuals in an artificial scenario and on whole-brain data.

Model synthesis is a new form of transfer learning and multitask learning. Transfer learning typically seeks to improve the performance of a machine learning model on one or more target domains or tasks by leveraging data from additional domains and tasks \cite{pan2009, weiss2016}, while multitask seeks models that simultaneously perform well on all domain and task pairs in the training data \cite{caruana1997}. Model synthesis shares these goals but goes further by seeking to learn models that can ultimately perform well when modeling tasks (experimental conditions) that are never observed in the training data for a particular domain (system instance). Moreover, model synthesis makes the key assumption that each system instance is an example of the same general type of system and therefore seeks a unified model structure that is sufficient to explain the data in all domains and tasks. This makes model synthesis particularly appropriate when seeking to learn unifying principles for various biological, physical, or social systems. 

DPMS is also distinct from, yet complementary to, recent applications of large language models (LLMs) in scientific domains. For example, some recent studies leverage the semantic knowledge and language understanding of LLMs to summarize literature, critique papers, and suggest new hypotheses \cite{luo2022, liang2024, skarlinski2024}. Pretrained LLMs can also guide experimental workflows \cite{boiko2023, lu2024, gottweis2025, ifargan2025, schmidgall2025}. These are forms of synthesis at the conceptual level, but they differ from the definition of synthesis explored in this work. 
Other related works train LLM-based, domain-specific, foundation models that can be applied to a variety of downstream problems \cite{chithrananda2020, irwin2022, cui2024, brixi2025genome}. The methods we present do not require the vast amounts of data required for foundation model pretraining. Nevertheless, future work could use a foundation model for the CPD in DPMS, simultaneously leveraging the synthetic and probabilistic properties of DPMS and the power of large pre-trained models. Finally, both LLMs and DPMS are designed to aid and surpass a human researcher’s ability to survey a diverse set of experimental results and identify common explanations that can account for findings across experiments. However, both humans and LLMs synthesize information suboptimally. For example, humans tend to overestimate confidence and hallucination remains a problem for current LLMs \cite{huang2025survey}. DPMS provides a theoretical approach and formal framework for ameliorating such failure modes.

Several neuroscience methods have been proposed for incorporating data from multiple brains into data-driven models. One typical approach is to register data from multiple individual brains to a common reference brain or atlas. Structure shared across brains can then be recognized and modeled at the level or voxels, neuron types, and brain regions \cite{naumann2016whole, marques2020internal, yang2022brainstem, brezovec2023neural, brezovec2024mapping}. DPMS goes beyond these methods by allowing spatial variability across individual brains and producing cellular-resolution models. Another approach combines data from many individual animals to learn a foundation model enabling generalization across individuals and data-efficient customization \cite{vermani2024meta, wang2025foundation, kaifosh2025generic}. While powerful in extending traditional machine learning approaches to large models and datasets, they only apply to specific model classes. DPMS is specifically designed to combine data across individuals under more general circumstances.
As with DPMS, auxiliary properties can help these traditional machine learning paradigms generalize across individuals \cite{lurz2021, schneider2023learnable}; in the context of our work, this would represent a form of non-probabilistic model synthesis. 

There are multiple promising ways to extend upon the work here presented. First, we here focused on demonstrating DPMS using low-dimensional properties. However, model synthesis can be applied to higher-dimensional spaces of continuous-valued properties, and developing neural networks parameterizing CPDs appropriate for such properties would be very valuable. Second, there are additional exciting directions DPMS may enable in neuroscience. For example, the ability to apply model-driven, cellular-resolution perturbations can aid in discovery of biological system dynamics \cite{emiliani2022optogenetics}, and DPMS could help fit these models given limited experimental data. Finally, while we focused on applications in neuroscience, model synthesis can be applied broadly to other complex systems \cite{newman2011complex}, such as learning descriptions of physical dynamic systems, genetic networks, or the behavior of financial markets.


\section{Methods}

\subsection{Accommodating different dimensionalities across system instances}
\label{sec:cpd}

When performing model synthesis, the dimensionality of both model parameters, $\theta^s$, and properties, $M^s$, will generally differ across example system instances, and we must design the CPD to accommodate this.  There are multiple ways in which this could be addressed.  For example, when the CPD is parameterized by a neural network, standard mechanisms to handle variable-dimensional input and output, such as pooling, attention, and convolution, could be employed. Throughout this paper, we employ the particularly simple approach of assuming a factorized CPD. In this approach, we assume that $\theta^s\in \mathbb{R}^{d^s_\theta\times m}$ and $M^s \in \mathbb{R}^{d^s_\theta\times r}$ are matrices with a number of rows, $d^s_\theta$, that can differ, but with a fixed number of columns across system instances.  Furthermore, we assume that individual rows of $M^s$ can be used to predict the corresponding rows of $\theta^s$.  For example, in a neuroscience setting, the rows of $\theta^s$ might represent neurons providing readout weights in a model, which we seek to predict from their individual genetic properties, represented in the rows of $M^s$. Assuming conditional independence, we then introduce a factorized CPD of the form $p_\gamma(\theta^s | M^s) = \prod_{i=1}^{d^s_\theta} p_\gamma(\theta^s[i,:]| M^s[i,:] )$, where the notation $[i,:]$ indicates the $i^{th}$ row of a matrix. This simplifies the problem of learning a CPD that accommodates variable dimensionality to one of learning a conditional distribution over individual rows of $\theta^s$ given individual rows of $M^s$, both of which are of fixed dimensions. Although the conditional independence assumption inherent in this approach may only approximately hold, this is a natural decomposition that can greatly simplify the problem of forming a CPD for many forms of machine learning models, including those underlying linear regression, logistic regression, or factor analysis.

\subsection{Approximating theoretical expectations with numerical sampling}
\label{sec:sampling_mds}

For many forms for the model and the CPD, the terms in $\mathcal{L}$ will be impossible to analytically calculate. In these cases, we approximate them by sampling \cite{williams1992, kingma2014, roeder2017, mnih2014}. For example, 
\begin{align}
    \mathbb{E}_{q_{\phi^s}(\theta^s)}\left[ \log p(Y^s|X^s, \theta^s) \right] &\approx \frac{1}{N}\sum_{i=1}^N \log p(Y^s|X^s, \tilde{\theta}^s_i),  \label{eq:ll_approx} \\
    \text{KL}\left[q_{\phi^s}(\theta^s) || p(\theta^s|M^s; \gamma) \right] &\approx \frac{1}{N} \sum_{i=1}^{N}\left( \log q_{\phi^s}(\tilde{\theta}^s_i) - \log p_\gamma(\tilde{\theta}^s_i|M^s)\right), \label{eq:kl_approx}
\end{align}
\noindent where $\tilde{\theta}^s_1, \ldots, \tilde{\theta}^s_{N}$ are i.i.d. samples drawn from $q_{\phi^s}(\theta^s)$. In practice, we use $N=1$ throughout.

\subsection{Synthesizing dimensionality reduction models}
\label{sec:dim_red_mdls}

To demonstrate that DPMS can be applied to many different types of models, we now describe how it can be used to synthesize a broad class of generative models underlying a multitude of dimensionality reduction methods. We consider dimensionality reduction models for observed data of the form $Y^s \in \mathbb{R}^{n^s \times d_y^s}$, where $n^s$ is the number of observed samples and $d_y^s$ the number of observed variables for system instance $s$. For each example system, we seek to estimate the latent state, $Z^s \in \mathbb{R}^{n^s \times d_z}$, associated with each observed sample, where $d_z$ is the dimensionality of the latent state. We choose to model the latent state as inhabiting the same latent space across system instances, allowing us to compare data from each in the same low-dimensional space, so $d_z$ is the same across system instances. 

We focus on synthesizing hierarchical models specified by two model components. Many models underlying different dimensionality reduction techniques, such as factor analysis  and probabilistic principal components analysis \cite{tipping1999}, and a variety of methods incorporating latent dynamical systems (LDS) (e.g, \cite{kalman1960, rabiner2002}) take on this two-component form. The first component, $p_\lambda(Z^s)$, specifies the prior distribution over $Z^s$, where $\lambda$ is a set of optional parameters we seek to learn. For factor analysis and probabilistic principal components analysis, this prior has no learnable parameters and simply specifies that each row of $Z^s$ is independently and identically distributed according to a standard multivariate normal distribution. Alternatively, when $Y^s$ represents time series data, $p_\lambda(Z^s)$ might represent the probability of latent state trajectories and $\lambda$ would represent the parameters of the underlying LDS that we believe is shared across system instances and we seek to learn. The second model component is an observation model, and specifies the conditional probability of observed data given the latent state, which we denote as $p(Y^s| Z^s, \theta^s)$ for a system instance $s$. Here, $\theta^s$ are again parameters that can vary in dimensionality across system instances.

We apply DPMS to synthesize models of this form as follows. First,  $\lambda$, the parameters of the prior over the latent state, are shared across system instances. We do not attempt to predict these from system instance properties and instead seek to learn a single point estimate for them. This is a specific example of a general approach for fitting parameters shared across system instances we present in detail in \Sect{practical_concerns}. In contrast to the prior over the latent state, we assume that the way in which the latent state manifests itself in the observed variables can differ across system instances, and for this reason, we seek to predict $\theta^s$, the parameters of the observation model mapping latent state to observed variables, from the system instance properties by learning a CPD, $p_\gamma(\theta^s|M^s)$. We then introduce approximate posterior distributions over $\theta^s$ and $Z^s$ for each system instance, $q_{\phi^s_\theta}(\theta^s)$ and $q_{\phi^s_z}(Z^s)$, where $\phi^s_\theta$ and $\phi^s_z$ are parameters we seek to learn, and optimize the ELBO, which as we show in \appendixSect{elbo_dr_derivation} now takes the following form
\begin{align}
 \sum_{s=1}^S \mathbb{E}_{q_{\phi^s_\theta, \phi^s_z}(Z^s, \theta^s)}\left[ \log p(Y^s|Z^s, \theta^s) \right] - \text{KL}\left[q_{\phi^s_\theta}(\theta^s) || p_\gamma(\theta^s|M^s)\right] -  \text{KL}\left[q_{\phi^s_z}(Z^s) || p_\lambda(Z^s)\right], \label{eq:ELBO_DR}
\end{align}
where we define $q_{\phi^s_\theta, \phi^s_z} := q_{\phi^s_\theta}(\theta^s)q_{\phi^s_z}(Z^s)$. We note that selecting approximate posteriors that factorize over $\theta^s$ and $Z^s$ is not strictly necessary. In principle, it would also be possible to apply DPMS with more general joint posteriors capable of modeling the correlation between $\theta^s$ and $Z^s$, and this is a design choice we have made simply for convenience. \Eq{ELBO_DR} still lower-bounds the log-likelihood of the data observed from all system instances, and using the techniques outlined in \Sect{DPMS}, we can optimize it with respect to $\gamma, \lambda$ and $\{\phi_{\theta}^s, \phi_z^s\}_{s=1}^S$ to estimate the CPD and the prior over latent state and fit the approximate posteriors over model parameters and latent state for each system instance.

\subsection{Sum of hyperrectangular basis functions (SHBF) functions}
\label{sec:SHBF_fcns}

We define the form of the functions we used to predict model parameters from properties in this work. We select a form appropriate for representing functions over relatively low-dimensional domains with a mapping to output that can potentially vary drastically and non-smoothly over different local regions of the domain. We represent such functions as the sum of tiled hyperrectangular basis functions, and for this reason refer to them as SHBF functions. We define SHBF functions as follows. We assume that properties can take on real values in a bounded $n$-dimensional hyperrectangle $\mathcal{R}$.  We then define a set $r_1, \ldots, r_I$ of smaller and potentially overlapping $n$-dimensional hyperrectangles that cover $\mathcal{R}$. We associate each $r_i$ with a learnable coefficient $c_i$ and define $f: \mathbb{R}^m \rightarrow \mathbb{R}$ as
\begin{align*}
    f(x) = \sum_{i=1}^I c_i\mathbb{I}(x \in r_i),
\end{align*}
where $\mathbb{I}$ is the indicator function.  Conceptually, $f$ breaks up space into a set of potentially overlapping hyperrectangles and assigns a value for $x$ by summing the coefficients associated with the hyperrectangles $m$ falls within. 

Careful attention to the way that the hyperrectangles, $r_i$, are laid out can enable extremely efficient implementations of these functions.  Specifically, by assigning the hyperrectangles, $r_i$, so that for a given dimension of $\mathcal{R}$ they all have the same width and have leading edges spaced at fixed intervals it is possible to directly calculate which hyperrectangles, $r_i$, a point, $x$, falls within.  This means computing the exhaustive sum above can be reduced to efficiently determining which hyperrectangles a point falls in and summing the small number of coefficients associated with these. Further, when using overlapping hyperrectangles, it is possible to add padding to the range that the hyperrectangles, $r_i$, are defined over to ensure that any $x \in \mathcal{R}$ falls within the same number of hyperrectangles.  This means these functions can be efficiently implemented in a tensor based way to process multiple inputs in parallel on modern GPU hardware, and we find in practice that functions of this form can be very computationally fast to work with.

\subsection{Relating only some model parameters through system properties}
\label{sec:practical_concerns}

To aid in the presentation of the core ideas, we have presented DPMS in its purest form to this point. However, it may not always be necessary or desirable to predict all model parameters from system properties, and there may be certain parameters we desire to rigidly share across models. Indeed, we have already seen an example of this in how we handled the parameters of the prior over the latent state in \Sect{dim_red_mdls}. We now further describe how DPMS can be generalized to accommodate these concerns by splitting the parameters, $\theta^s$, for a model for an individual system instance into three sets (\Fig{models}).

The first set of parameters, $\theta^s_\text{props}$, is the set of parameters that we do seek to predict from the system properties. We learn to predict $\theta^s_\text{props}$ from measurable properties with the CPD, which we now denote as $p_\gamma(\theta^s_\text{props}|M^s)$. We introduce an approximate posterior over $\theta^s_\text{props}$ for each example system, which we refer to as $q_{\phi^s_\text{props}}(\theta^s_\text{props})$, where $\phi^s_\text{props}$ are a set of parameters that we will optimize.

The second set of parameters, $\theta_\text{shared}$, is a set of parameters that are shared and therefore take on the same value across system instances. These correspond to the parameters of the prior over latent state for dimensionality reduction models in \Sect{dim_red_mdls}. However, these can also arise if models for system instances share a ``core component.'' As we explore further in our results in the context of regression models, this core component might be a deep neural network that maps input data to a common low-dimensional space before output variables are predicted for each system instance (\Fig{models}a). For simplicity, we choose to learn point estimates for $\theta_\text{shared}$ in this work, though it would be possible to introduce learnable prior and posterior distributions over these parameters as well.

Finally, the third set of parameters, $\theta^s_\text{no-props}$, are those that can vary across system instances, but that we do not seek to predict from the system's properties. This might be because we suspect that the properties we have access to will be uninformative for these parameters. We assign non-conditional priors, potentially with their own learnable parameters, $\delta$, over these parameters, $p_{\delta}(\theta^s_\text{no-props})$. In this work, we use the same priors across system instances, represented as $p_{\delta}(\theta^s_\text{no-props})$, requiring the dimensionality of $\theta^s_\text{no-props}$ to be the same across system instances. However, this could be easily generalized to allow for priors specific to each example system, permitting the dimensionality of $\theta^s_\text{no-props}$ to vary. We form approximate posteriors over these parameters for each example system, which we denote as $q_{\phi^s_\text{no-props}}(\theta^s_\text{no-props})$, where $\phi^s_\text{no-props}$ are again parameters we optimize. To avoid confusion, we note that the idea of learning the parameters over a prior may seem odd if viewed from a Bayesian perspective. However, when the prior is viewed as describing a component of a hierarchical generative model, e.g., one describing how the world generates $\theta^s_\text{no-props}$, then learning $\delta$ can be understood simply as a natural part of fitting a hierarchical probabilistic model.

We now present the ELBO for synthesizing models with these sets of parameters. For brevity, we present the ELBO for classification and regression models. However, the ELBO for dimensionality reduction models of the forms in \Sect{dim_red_mdls} is nearly identical and simply incorporates an additional KL term between the prior and posteriors over latent state.  Distinguishing the three sets of parameters just defined, we now refer to the probability of observed data for system instance $s$ under a model for classification or regression with the notation $p_{\theta_\text{shared}}(Y^s|X^s, \theta^s_\text{props}, \theta^s_\text{no-props})$, and the ELBO now takes the form
\begin{align}
 \sum_{s=1}^S \mathbb{E}_{q_{\phi^s_\text{props}, \phi^s_\text{no-props}}(\theta^s_\text{props}, \theta^s_\text{no-props})}&\left[ \log p_{\theta_\text{shared}}(Y^s|X^s, \theta^s_\text{props}, \theta^s_\text{no-props}) \right] - \text{KL}\left[q_{\phi^s_\text{props}}(\theta^s_\text{props}) || p_\gamma(\theta^s_\text{props}|M^s)\right] \nonumber \\
 - &\text{KL}\left[q_{\phi^s_\text{no-props}}(\theta^s_\text{no-props}) || p_\delta(\theta^s_\text{no-props})\right], \label{eq:ELBO_FULL}
\end{align}
where we define $q_{\phi^s_\text{props}, \phi^s_\text{no-props}}(\theta^s_\text{props}, \theta^s_\text{no-props}) := q_{\phi^s_\text{props}}(\theta^s_\text{props})q_{\phi^s_\text{no-props}}(\theta^s_\text{no-props})$. The derivation on \Eq{ELBO_FULL} can be found in \appendixSect{full_elbo_derivation}. As in \Sect{dim_red_mdls}, for convenience we have made the assumption that the posterior over model parameters for each system instance factorizes, and again note that in principle this could be relaxed. 

\subsection{Application details for DPMS}

\subsubsection{Application details for Section \ref{sec:syn_example}}
\label{sec:syn_mdl_fitting}

Our goal is to apply DPMS to learn models for the simulated brains. For each individual $s$, we suppose that the activity from $d^s_x$ neurons at time $t$, $x^s_t \in \mathbb{R}^{d_x^s}$, drives behavior, $y^s_t \in \mathbb{R}$, according to
\begin{align*}
    y_t^s &= f(l_t^s) + r_t^s \\
    l_t^s &= ({\omega^s})^T x_t^s \\
    r_t^s &\sim \mathcal{N}(0, (\nu^s)^2) \\
    \nu^s &\sim \Gamma(\alpha=10, \beta=1000),
\end{align*}
\noindent where $T$ denotes transpose, $f(l^s_t) = \sin(3l^s_t) + l^s_t$ represents the conserved mapping from $l^s_t \in \mathbb{R}$ to behavior, $\omega^s \in \mathbb{R}^{d^s_x}$ are individual-specific weights for the projection into the shared low-dimensional space, $r_t^s \in \mathbb{R}$ is recording noise of standard deviation $\nu^s \in \mathbb{R}$, and $\mathcal{N}$ and $\Gamma$ denote the normal and gamma distributions. Values of $\omega^s$ were randomly drawn for each brain from Normal distributions conditioned on neuron position, with the conditional mean and standard deviations formed from a sum of random bump functions. DPMS amounts to forming posteriors over $\omega^s$ and $\nu^s$ for each individual, learning the shared function $f$, learning the functions $\mu$ and $\sigma$ of the CPD over $\omega^s$, and learning the prior over $\nu^s$. We generated out-of-distribution data by allowing all neurons to be active and $l^s_t$ to explore the full domain of $f$ and quantified performance by measuring the ELBO, R-squared, and correlation between the true and predicted behavior.  Additional details on the generation of the ground-truth simulated brains and data as well as metrics can be found in the supplement. 

Applying the strategy outlined in \Sect{practical_concerns}, we let the weights depend on properties through the CPD. Defining hats to indicate estimated entities, we specify
\begin{align*}
    p_\gamma(\omega^s| M^s) &= \prod_{i=1}^{d^s_x}p_\gamma(\omega^s[i]| M^s[i,:]) = \prod_{i=1}^{d^s_x}\mathcal{N}(\hat{\mu}_{\gamma_1}(M[i,:]), \hat{\sigma}^2_{\gamma_2}(M[i,:])) 
\end{align*}
\noindent where $\gamma = \{\gamma_1, \gamma_2\}$ are parameters we seek to learn. Here $\hat{\mu}_{\gamma_1}$ and $\hat{\sigma}_{\gamma_2}$ assign values over finely spaced overlapping cubes in space (see \Sect{SHBF_fcns}). While a simple class of functions, we found these piecewise constant functions outperformed other alternatives in their application to real neural data, such as smooth feedforward neural networks. This functional form also implicitly discretizes property space, so that \Cra{main} can be directly appealed to for understanding how they will be optimized. The prior, $p_\delta(\nu^s)$, we learn over $\nu^s$ is a Gamma distribution with learnable shape and rate parameters. We specify the form of $\hat{f}_{\theta_\text{shared}}$ as a general feed-forward neural network. Finally, we specify the form of the approximate posteriors using a Gaussian mean-field approximation for $q_{\phi^s_\text{props}}(\theta^s)$ and a Gamma distribution for $q_{\phi^s_\text{no-props}}(\nu^s)$.  

With all of this specified, we followed the strategy outlined in \Sect{DPMS} for optimizing \Eq{ELBO_FULL}. To prevent overfitting, early-stopping was performed based on a small amount of validation data, generated in the same was as the training data, for each individual. Additional details regarding model components and fitting can be found in the supplement. 

We wanted to compare models synthesized with DPMS to those fit to individual system instances. To achieve an apples-to-apples comparison, we applied the full DPMS framework when fitting models to data from one individual alone, and we again applied early stopping based on validation data. We note that because the CPD factorizes over neurons, it is still in principle possible to learn a valid CPD given sufficiently dense sampling of the behavior and property space by a single individual. Also note that to compare fit and ground-truth quantities, we applied scales and offsets that account for non-identifiability in the model class (see supplement for details). 

\subsubsection{Application details for section \ref{sec:regResults}}
\label{sec:app_details_reg}

From the datasets described in Section \ref{sec:data_preprocessing}, we identified 7 fish that were imaged at similar rates and that demonstrated robust swimming responses to phototaxis. For 4 of these fish, the entire recording duration included phototaxis conditions, and we treated these as the base fish. As only a portion of the recording time was devoted to phototaxis for the remaining three fish,  we treated these as the target fish.

The form of regression models we employ is shown in \Fig{models}a, in which the activity of all recorded neurons from fish $s$ at one time step, $x^s_t \in \mathbb{R}^{d^s_x}$, is used to predict swimming signals, $y^s_t \in \mathbb{R}^2$, at the next. We fit two-component models nearly identical to those fit in \Sect{syn_mdl_fitting} with two small differences. First, we use a shared low-dimensional space with ten dimensions, which we found to produce better model performance than a one-dimensional latent space. Second, the same general form of neural network was used for $\hat{f}$ as in \Sect{syn_mdl_fitting}, but adjustments were made to accommodate the larger dimensionalities of the low-dimensional space and the predicted output. We then select the forms of the CPD, fixed priors and approximate posteriors as follows.  Following the approach in \Sect{practical_concerns}, for these models,  $\theta^s_\text{props} \in \mathbb{R}^{d^s_x \times 10}$ are the weights, $\omega^s \in \mathbb{R}^{d_x^s \times 10}$, mapping from neural activity to the common low-dimensional space. We use a factorized CPD, 
\begin{align*}
    p_\gamma(\omega^s| M^s) = \prod_{j=1}^{10} \prod_{i=1}^{d_x^s} \mathcal{N}(\mu_{\gamma_{1,j}}(M[i,:]), \sigma_{\gamma_{2,j}}(M[i,:])),
\end{align*}
where $\mu_j$ and $\sigma_j$ are functions learned for each dimension, $j$, of the same form as those in \Sect{syn_mdl_fitting} and $\gamma = \{\gamma_{1,j}, \gamma_{2,j}\}_{j=1}^{10}$ are the parameters of the CPD.  The parameters $\theta^s_\text{no-props}$ are the standard deviation of the noise for the swim signals. We learn a Gamma prior over these. Finally, we used a mean field approximation for both $q(\omega^s)$ and  $q(\theta^s_\text{no-props})$, composed of a product of univariate Gaussian and Gamma distributions.
We again follow the strategy outlined in Section \ref{sec:practical_concerns} for optimizing \Eq{ELBO_FULL}.

Following the same logic outlined in \Sect{syn_mdl_fitting}, we still applied the full synthesis framework when fitting models to data from the target fish alone to ensure an apples-to-apples comparison. 

To avoid overfitting in both scenarios, we applied early stopping based on held-out validation data. We divided the data for each fish into equally sized sets of train, validation and test data, and varied the data assigned to each with 3-fold cross validation, being careful to roughly balance the swim vigor across train, validation and test sets. Early stopping was performed on the validation data for all the base fish and target fish when applying DPMS and for the target fish alone when fitting models to it in isolation. The validation and test data for the target fish was always the same in both scenarios. 

To adjust for different amounts of testing data available across fish, we divided the ELBO by the number of samples in the testing data to arrive at the normalized ELBO. 

Finally, to further examine the dependence of the fit models on the amount of data available for the target fish, we varied the percentage of train and validation data allocated to each fold actually used for the target fish.  Additional details on the form of $\hat{f}$, CPD, fixed priors, approximate posteriors and model fitting are provided in the supplement. 
 
\subsubsection{Application Details for Section \ref{sec:dimRedResults}}
\label{sec:dr_app_details}

The form of factor analysis models is shown in \Fig{models}b. In these models, brain-wide neural activity at time $t$, $x^s_t \in \mathbb{R}^{d_x^s}$, for fish $s$ is explained by a small number of latent state variables, $z^s_t \in \mathbb{R}^{d_z}$. We model latent state variables as inhabiting the same latent space for all fish, so $d_z$ is fixed across fish. Under a factor analysis model, the observed neural activity for fish $s$ is modeled as 
\begin{align*}
    x_t^s &= \Lambda^s z_t^s + \eta^s + r_t^s \\
    z_t^s &\sim \mathcal{N}(0, I) \\
    r_t^s &\sim \mathcal{N}(0, \text{diag}[(\nu^s)^2]),
\end{align*}
for $\Lambda^s \in \mathbb{R}^{d^s_x \times d_z}$, $\eta^s \in \mathbb{R}^{d^s_x}$ and where $\nu^s \in \mathbb{R}^{d^s_x}$ is a vector of standard deviations so $\text{diag}[(\nu^s)^2])$ is a diagonal covariance matrix. We found that using a latent space with ten dimensions worked well, and set $d_z = 10$.

The key intuition behind our approach is that we can use measurable properties to find consistent mappings between latent and observed variables across system instances. These mappings implicitly define the latent space for each example system, so by ensuring these mappings are consistent, we ensure that latent state variables can indeed be interpreted as residing in the same latent space across system instances. Intuitively, through synthesis we can learn that neurons at certain positions in the brain couple to certain latent computational quantities in particular ways, and use this information to relate models across fish. 

As discussed in \Sect{DPMS}, a CPD that relates system properties to model parameters is required to enable this synthesis. Applying \Sect{dim_red_mdls}, $\theta^s = \{\Lambda^s, \eta^s, \nu^s\}$, and we learn a CPD of the form 
\begin{align*}
 p_\gamma(\Lambda^s, \eta^s, \nu^s | M^s) = p_{\gamma_1}(\Lambda^s|M^s)p_{\gamma_2}(\eta^s|M^s)p_{\gamma_3}(\nu^s| M^s),
 \end{align*}
 where $\gamma = \{\gamma_1, \gamma_2, \gamma_3\}$, $p_{\gamma_1}(\Lambda^s|M^s)$ and $p_{\gamma_2}(\eta^s|M^s)$ are again a product of conditional Gaussian distributions with mean and standard deviations that are learnable functions of neuron position, just as $p(\omega^s| M^s)$ in \Sect{regResults}, and $p_{\gamma_3}(\nu^s| M^s) = \prod_{i=1}^{d^s}\text{Gamma}(\alpha_{\gamma_{3, 1}}(M[i,:], \beta_{\gamma_{3, 2}}(M[i,:])))$, where $\alpha$ and $\beta$ are functions that assign fixed values over finely spaced cubes in space with parameters $\gamma_3 = \{\gamma_{3, 1}, \gamma_{3, 2}\}$.

We use approximate posteriors for each fish, $q_{\phi^s_\theta}(\Lambda^s, \eta^s, \nu^s)$, that are a product of univariate Gaussian and Gamma distributions, and we define $q_{\phi_z^s}(Z^s) = \prod_{i=1}^{n^s}\mathcal{N}(\iota^s_i, \Sigma^s)$, where $\phi_z^s = \{ \{\iota^s_i\}_{i=1}^s, \Sigma^s\}$ and $\iota^s_i \in \mathbb{R}^k$ are mean vectors for each data point and $\Sigma^s \in \mathbb{R}^{k \times k}$ is a full covariance matrix shared across data points. Sharing the covariance across data points reduces the number of parameters to fit and is also motivated by the observation that posteriors over the latent variables for all data points for factor analysis models with fixed parameters share the same covariance.   We fit models with stochastic gradient ascent as described in \Sect{DPMS} and use early stopping based on the performance of all fish to avoid overfitting. Additional details on the form of CPD and approximate posteriors and model fitting and evaluation can be found in the supplement.

\subsection{Preprocessing of experimental data}
\label{sec:data_preprocessing}

The analyses in \sects{regResults}{dimRedResults} used publicly available data from \cite{chen2018}.  Minimal additional processing of data was done.  In particular, neural activity, recorded through whole-brain calcium imaging and reported as $\Delta\text{F}/\text{F}$ in the released data was scaled by a factor of $10000$ for analyses in \sect{regResults} and $10$ for analyses in \sect{dimRedResults}.  This was done to improve the numerical stability of model fitting.

Additionally, for each fish, swimming signals recorded in the form of the smoothed, local standard deviation of voltage signals from motor nerves running down both sides of a fish's tail in the original data, were scaled. This was done because the magnitude of these signals could vary from individual to individual for non-biological reasons (e.g., electrode impedance). To account for this, we normalized the swim signals for each fish by dividing them by the maximum value recorded on either electrode during phototaxis.  We chose to normalize by the largest value during phototaxis because swimming behavior can change across conditions, and while fish were recorded under different conditions, at least one portion of each experiment for each fish was devoted to phototaxis.  After this scaling was performed, swim signals were multiplied by a factor of 100, again to improve numerical stability during fitting. 

The exact set of fish used in each analysis, and the criteria for selecting them, are provided in the respective sections of the Methods and Supplemental Methods.

\subsection{Data and code availability}

All datasets used in this study were either generated directly by the code in the accompanying GitHub repository and can be qualitatively reproduced, or were obtained from previously published, publicly available sources. The larval zebrafish whole-brain imaging data of Chen*, Mu*, Hu* et al. (Neuron, 2018) can be downloaded from the Janelia Research Campus repository (\nolinkurl{https://doi.org/10.25378/janelia.7272617})

All code used to implement DPMS and perform the analyses in this study is available at \nolinkurl{https://github.com/neuro-will/probabilistic_model_synthesis}.

\section{Acknowledgements}

This work was supported by the Howard Hughes Medical Institute. JEF additionally acknowledges support from the National Institute for Theory and Mathematics in Biology through the National Science Foundation (grant number DMS-2235451) and the Simons Foundation (grant number MPTMPS-00005320).

{\small
\newpage
 \bibliographystyle{unsrt}
 \bibliography{BishopPapers.bib}
 }

\newpage
\setcounter{page}{1}
\appendix

\renewcommand{\thesection}{S\arabic{section}}
\renewcommand{\thetable}{S\arabic{table}}
\renewcommand{\thefigure}{S\arabic{figure}}

\section{Derivations for various forms of the ELBO}
\label{sec:ELBO_all}

In this section, we will derive general forms of evidence lower bounds (ELBO) applicable to synthesizing classification, regression and dimensionality reduction models, and will show that when the evidence lower bound is tight, the approximate posteriors must equal the true posteriors. Together these results provide a basis for understanding why maximizing the ELBO provides a means for estimating the parameters of the CPD, the shared parameters when they appear, and the parameters of the approximate posteriors.
We will then use these results to derive the specific forms of the ELBO for  \Eq{ELBO}, \Eq{ELBO_DR} and \Eq{ELBO_FULL}.

\subsection{General forms of the ELBO}

We begin by reviewing the basics of variational inference appropriate for our setting and derive general forms of the ELBO from which we can later derive the specific forms in the main text. In doing so, we will show that the ELBO forms a lower bound on the log-likelihood of observed data and that when it is tight, the approximate posteriors must equal the true posteriors. The review presented here is provided for completeness and largely follows derivations available elsewhere (e.g., \cite{blei2017}). 

Consider generative models with the following two components.  The first is a conditional prior over latent variables, $\mathbf{L}$, given properties, $\mathbf{M}$, which we denote as $p_\alpha(\mathbf{L} | \mathbf{M})$, where $\alpha$ are parameters we may seek to learn.  In this review, we use bold fonts to distinguish the variables we introduce here from those in the main text. When we derive each version of the ELBO below, we will map the variables in the main text to these. The second component provides the probability of some $\mathbf{Y}$ given some $\mathbf{X}$ and $\mathbf{L}$, which we denote as $p_\beta(\mathbf{Y} |\mathbf{X}, \mathbf{L})$, where $\beta$ are again parameters we may seek to learn. We introduce $\mathbf{X}$ so the derivations here can be directly mapped onto the problem of synthesizing classification and regression models.  We will then describe how this analysis can be modified for dimensionality reduction models. 

We introduce an approximate posterior distribution $q_\mathbf{\rho}(\mathbf{L})$, where $\rho$ are parameters we will optimize so the KL-divergence 
between the approximate and true posterior, $p_{\alpha, \beta}(\mathbf{L} |\mathbf{X}, \mathbf{Y}, \mathbf{M})$, is minimized. Note that both $\alpha$ and $\beta$ appear as parameters of the posterior since they are the parameters of the likelihood and prior. We can write this KL-divergence as
\begin{align}
    \text{KL}\left[q_\mathbf{\rho}(\mathbf{L}) || p_{\alpha, \beta}(\mathbf{L} |\mathbf{X}, \mathbf{Y}, \mathbf{M}) \right] &= \mathbf{E}_{q_\mathbf{\rho}(\mathbf{L})}\left( \log q_\mathbf{\rho}(\mathbf{L}) - \log p_{\alpha, \beta}(\mathbf{L} |\mathbf{X}, \mathbf{Y}, \mathbf{M})\right) \notag \\
    &= \mathbf{E}_{q_\mathbf{\rho}(\mathbf{L})}\left( \log q_\mathbf{\rho}(\mathbf{L}) - \log \frac{p_\beta(\mathbf{Y} |\mathbf{X}, \mathbf{L})p_\alpha(\mathbf{L} | \mathbf{M})}{p_{\alpha, \beta}(\mathbf{Y}| \mathbf{X}, \mathbf{M})}\right) \notag \\
    &= - \mathbf{E}_{q_\mathbf{\rho}(\mathbf{L})}\left(\log p_\beta(\mathbf{Y} |\mathbf{X}, \mathbf{L})\right) + \text{KL}\left[q_\mathbf{\rho}(\mathbf{L}) || p_\alpha(\mathbf{L} | \mathbf{M}) \right] & \notag \\ &\quad\quad\quad +\log p_{\alpha, \beta}(\mathbf{Y}| \mathbf{X}, \mathbf{M}). \label{eq:ELBO_derivation}
\end{align}
Note that $\log p_{\alpha, \beta}(\mathbf{Y}| \mathbf{X}, \mathbf{M})$ is the conditional log-likelihood of observed data. Recall that KL divergence is strictly non-negative, so from \Eq{ELBO_derivation} we can conclude
\begin{align}
    \log p_{\alpha, \beta}(\mathbf{Y}| \mathbf{X}, \mathbf{M}) \geq \mathbf{E}_{q_\mathbf{\rho}(\mathbf{L})}\left(\log p_\beta(\mathbf{Y} |\mathbf{X}, \mathbf{L})\right) - \text{KL}\left[q_\mathbf{\rho}(\mathbf{L}) || p_\alpha(\mathbf{L} | \mathbf{M}) \right]. \label{eq:general_ELBO}
\end{align}
The right hand side lower bounds the conditional log-likelihood of observed data, and is the general form of the ELBO we can apply when synthesizing classification and regression models in this work.  Having established this, we now also show that when this bound is tight, the approximate posterior must equal the true posterior.  This can be established by noting that when the bound is tight, the left and right sides of \Eq{general_ELBO} are equal.  Substituting this into \Eq{ELBO_derivation} we then have $\text{KL}\left[q_\mathbf{\rho}(\mathbf{L}) || p_{\alpha, \beta}(\mathbf{L} |\mathbf{X}, \mathbf{Y}, \mathbf{M}) \right] = 0$. KL divergence is zero only when the two distributions are equal, establishing that when the ELBO is tight, the approximate posterior must equal the true posterior. 

The same analysis can be repeated for dimensionality reduction models of the form introduced in \Sect{dim_red_mdls}.  Here, we introduce latent state $\mathbf{Z}$, a prior over latent state $p_\zeta(\mathbf{Z})$, where $\zeta$ are parameters we may seek to learn, and the observation model $p_\beta(\mathbf{Y} | \mathbf{Z}, \mathbf{L})$. The most straightforward way of extending the above is by assuming we can analytically marginalize out the latent state to calculate the likelihood of observed data conditioned on $\mathbf{L}$, $p_{\beta, \zeta}(\mathbf{Y} | \mathbf{L})$. This is the most straightforward approach because it simply corresponds to dropping $\mathbf{X}$ in the above derivations. However, arriving at an analytic form for $p_{\beta, \zeta}(\mathbf{Y} | \mathbf{L})$ will often not be possible.  Therefore, a more general approach is to introduce an additional approximate posterior over $\mathbf{Z}$, $q_\omega(\mathbf{Z})$, and then seek to minimize the KL divergence between the joint approximate posterior for $\mathbf{L}$ and $\mathbf{Z}$, $q_\rho(\mathbf{L})q_\omega(\mathbf{Z})$, and the true posterior $p_{\alpha, \beta, \zeta }(\mathbf{L}, \mathbf{Z} | \mathbf{Y}, \mathbf{M})$, which can be written as
\begin{align*}
    \text{KL}\left[q_\rho(\mathbf{L})q_\omega(\mathbf{Z}) || p_{\alpha, \beta, \zeta }(\mathbf{L}, \mathbf{Z} | \mathbf{Y}, \mathbf{M})\right] &= \mathbb{E}_{q_\rho(\mathbf{L})q_\omega(\mathbf{Z})}\left(\log q_\rho(\mathbf{L})q_\omega(\mathbf{Z}) - \log p_{\alpha, \beta, \zeta }(\mathbf{L}, \mathbf{Z} | \mathbf{Y}, \mathbf{M}) \right) \notag \\
    &= \mathbb{E}_{q_\rho(\mathbf{L})q_\omega(\mathbf{Z})}\biggl(\log q_\rho(\mathbf{L})q_\omega(\mathbf{Z}) \notag \\
    &\quad\quad\quad -\log \frac{p_\beta(\mathbf{Y}|\mathbf{Z}, \mathbf{L})p_\alpha(\mathbf{L} | \mathbf{M})p_\zeta(\mathbf{Z})}{p_{\alpha, \beta, \zeta}(\mathbf{Y}|\mathbf{M})} \biggr) \notag \\
    &= - \mathbb{E}_{q_\rho(\mathbf{L})q_\omega(\mathbf{Z})}\left(\log p_\beta(\mathbf{Y}|\mathbf{Z}, \mathbf{L}) \right) + \text{KL}\left[q_\rho(\mathbf{L}) | p_\alpha(\mathbf{L} || \mathbf{M}) \right]  \notag \\
    &\quad\quad\quad +\text{KL}\left[q_\omega(\mathbf{Z}) || p_\zeta(\mathbf{Z}) \right] + \log p_{\alpha, \beta, \zeta}(\mathbf{Y}|\mathbf{M}).
\end{align*}
Following the same analysis as the above, we can then conclude that 
\begin{align}
    \log p_{\alpha, \beta, \zeta}(\mathbf{Y}|\mathbf{M}) \geq \mathbb{E}_{q_\rho(\mathbf{L})q_\omega(\mathbf{Z})}\left(\log p_\beta(\mathbf{Y}|\mathbf{Z}, \mathbf{L}) \right) - \text{KL}\left[q_\rho(\mathbf{L}) | p_\alpha(\mathbf{L} || \mathbf{M}) \right] - \text{KL}\left[q_\omega(\mathbf{Z}) || p_\zeta(\mathbf{Z}) \right]. \label{eq:general_dr_ELBO}
\end{align}
The right hand side of \Eq{general_dr_ELBO} is the general form of the ELBO we will use when synthesizing dimensionality reduction models.  Similar to above, we can show that when the bound is tight the approximate posteriors must equal the true posteriors. Having introduced these general forms for the ELBO, we now use them to derive \Eq{ELBO}, \Eq{ELBO_DR} and \Eq{ELBO_FULL}

\subsection{Derivation of \Eq{ELBO}}
\label{sec:basic_elbo_derivation}

We first derive \Eq{ELBO}, the form of the ELBO for synthesizing classification and regression models when all model parameters are related through system instance properties through the CPD.  Relating this to \Eq{general_ELBO} above, in this scenario $\mathbf{Y} = \{Y^s\}_{s=1}^S$ and $\mathbf{X} = \{X^s\}_{s=1}^S$, $\mathbf{L} = \{\theta^s\}_{s=1}^S$. We then have 
\begin{align}
    \log p_\beta(\mathbf{Y} |\mathbf{X}, \mathbf{L}) = \log p(\{Y^s\}_{s=1}^S | \{X^s, \theta^s\}_{s=1}^S) = \log \prod_{s=1}^S p(Y^s | X^s, \theta^s) = \sum_{s=1}^S \log p(Y^s | X^s, \theta^s),
\end{align}
where the third equality follows from the conditional independence of $Y^s$ given $X^s$ and $\theta^s$ across system instances, and we have dropped $\beta$ as there are no shared parameters in the models of individual system instances in this scenario.  If we continue by corresponding the parameters of the CPD, $\gamma$, with $\alpha$ above, it follows that  
\begin{align}
    p_\alpha(\mathbf{L} | \mathbf{M}) = p_\gamma(\{\theta^s\}_{s=1}^S | \{M^s \}_{s=1}^S) = \prod_{s=1}^S p_\gamma(\theta^s | M^s),
\end{align}
where the second step follows from the conditional independence of model parameters across system instances given properties under the CPD.  Additionally, under the scenario presented in \Sect{DPMS}, $q_\mathbf{\rho}(\mathbf{L}) = \prod_{s=1}^S q_{\phi^s}(\theta^s)$, where $\rho = \{\phi^s\}_{s=1}^S$. Finally, we have  $\text{KL}\left[\prod_{s=1}^S q_{\phi^s}(\theta^s) || \prod_{s=1}^S p_\gamma(\theta^s | M^s) \right] = \sum_{s=1}^S \text{KL}\left[ q_{\phi^s}(\theta^s) || p_\gamma(\theta^s | M^s) \right]$. Substituting the above into \Eq{general_ELBO}, it immediately follows that the ELBO takes the form 
\begin{align}
 \sum_{s=1}^S \mathbb{E}_{q_{\phi^s}(\theta^s)}\left[ \log p(Y^s|X^s, \theta^s) \right] - \text{KL}\left[q_{\phi^s}(\theta^s) || p_\gamma(\theta^s|M^s) \right]. 
\end{align}

\subsection{Derivation of \Eq{ELBO_DR}}
\label{sec:elbo_dr_derivation}

Using \Eq{general_dr_ELBO}, we now derive the form of the ELBO for synthesizing dimensionality reduction models as presented in \Sect{dim_red_mdls}.  Corresponding the variables in \Sect{dim_red_mdls} to \Eq{general_dr_ELBO}, we have $\mathbf{Y} = \{Y^s\}_{s=1}^S$, $\mathbf{Z} = \{Z^s\}_{s=1}^S$ and $\mathbf{M} = \{M^s\}_{s=1}^S$.  We also have 
\begin{align}
    \log p_\beta(\mathbf{Y} |\mathbf{Z}, \mathbf{L}) = \log p(\{Y^s\}_{s=1}^S | \{Z^s, \theta^s\}_{s=1}^S) = \log \prod_{s=1}^S p(Y^s | Z^s, \theta^s) = \sum_{s=1}^S \log p(Y^s | Z^s, \theta^s),
\end{align}
where we have again used the independence of observations across system instances conditioned on $Z^s$ and $\theta^s$ and dropped $\beta$. We can correspond $p_\alpha(\mathbf{L} | \mathbf{M})  = \prod_{s=1}^S p_\gamma(\theta^s | M^s)$, where $\alpha = \gamma$,  just as in \Sect{basic_elbo_derivation}.  We also have 
\begin{align*}
    p_\zeta(\mathbf{Z}) = p_\lambda(\{Z^s\}_{s=1}^S) = \log \prod p_\lambda(Z^s) = \sum_{s=1}^S \log p_\lambda(Z^s), 
\end{align*}
where $\zeta = \lambda$.  Finally, we recognize that $q_\mathbf{\rho}(\mathbf{L}) = \prod_{s=1}^S q_{\phi_\theta^s}(\theta^s)$ and $q_\omega(\mathbf{Z}) = \prod_{s=1}^S q_{\phi_z^s}(Z^s)$, where $\rho = \{\phi_\theta^s\}_{s=1}^S$ and $\omega = \{\phi_z^s\}_{s=1}^S$. Substituting the above in into the right hand side of \Eq{general_dr_ELBO}, we conclude the ELBO in this scenario is 
\begin{align*}
    \sum_{s=1}^S \mathbb{E}_{q_{\phi^s_\theta, \phi^s_z}(Z^s, \theta^s)}\left[ \log p(Y^s|Z^s, \theta^s) \right] - \text{KL}\left[q_{\phi^s_\theta}(\theta^s) || p_\gamma(\theta^s|M^s)\right] -  \text{KL}\left[q_{\phi^s_z}(Z^s) || p_\lambda(Z^s)\right],
\end{align*}
where we define $q_{\phi^s_\theta, \phi^s_z} := q_{\phi^s_\theta}(\theta^s)q_{\phi^s_z}(Z^s)$, and we recognize $\text{KL}\left[\prod_{s=1}^S q_{\phi_\theta^s}(\theta^s) || \prod_{s=1}^S p_\gamma(\theta^s | M^s) \right] = \sum_{s=1}^S \text{KL}\left[ q_{\phi_\theta^s}(\theta^s) || p_\gamma(\theta^s | M^s) \right]$, and $\text{KL}\left[\prod_{s=1}^S q_{\phi_z^s}(Z^s) || \prod_{s=1}^S p_\lambda(Z^s) \right] = \sum_{s=1}^S \text{KL}\left[ q_{\phi_z^s}(Z^s) || p_\lambda(Z^s) \right]$.

\subsection{Derivation of \Eq{ELBO_FULL}}
\label{sec:full_elbo_derivation}

We now derive the ELBO when synthesizing regression and classification models when model parameters are partitioned as in \Sect{practical_concerns}. Relating this to \Eq{general_ELBO} above, in this scenario $\mathbf{Y} = \{Y^s\}_{s=1}^S$ and $\mathbf{X} = \{X^s\}_{s=1}^S$. Now the model parameters we treat as latent variables fall into two classes, $\theta^s_\text{props}$ and $\theta^s_\text{no-props}$, so $\mathbf{L} = \{\theta^s_\text{props}, \theta^s_\text{no-props}\}_{s=1}^S$.  In addition, there are shared parameters across models for system instances, $\theta_\text{shared}$, we seek to learn point estimates for.  With this, we can identify 
\begin{align}
    \log p_\beta(\mathbf{Y} |\mathbf{X}, \mathbf{L})  = \sum_{s=1}^S \log p_{\theta_\text{shared}}(Y^s | X^s, \theta_\text{props}^s, \theta^s_\text{no-props}),
\end{align}
where we have identified $\beta = \theta_\text{shared}$ and used the same conditional independence properties as in \Sect{basic_elbo_derivation}.  We can also recognize that 
\begin{align}
    p_\alpha(\mathbf{L} | \mathbf{M}) = p_\gamma(\{\theta^s_\text{props}\}_{s=1}^S | \{M^s \}_{s=1}^S)p_\delta(\{\theta^s_\text{no-props}\}_{s=1}^S) = \prod_{s=1}^S p_\gamma(\theta^s_\text{props} | M^s)p_\delta(\theta^s_\text{no-props}),
\end{align}
where $\alpha = \{\lambda, \delta\}$, and we have used the conditional independence properties of $\theta^s_\text{props}$ and $\theta^s_\text{no-props}$.  Finally, we can identify 
\begin{align}
    q_\rho(\mathbf{L}) = \prod_{s=1}^S q_{\phi_\text{props}^s}(\theta_\text{props}^s)q_{\phi_\text{no-props}^s}(\theta_\text{no-props}^s), 
\end{align}
where $\rho = \{\phi_\text{props}^s, \phi_\text{no-props}^s\}_{s=1}^S$. We also have that
\begin{align*}
    &\text{KL}\left[\prod_{s=1}^S q_{\phi_\text{props}^s}(\theta_\text{props}^s)q_{\phi_\text{no-props}^s}(\theta_\text{no-props}^s) ||  \prod_{s=1}^S p_\gamma(\theta^s_\text{props} | M^s)p_\delta(\theta^s_\text{no-props})\right]  = \\
    &\quad\quad\quad\quad \sum_{s=1}^S \text{KL}\left[q_{\phi^s_\text{props}}(\theta^s_\text{props}) || p_\gamma(\theta^s_\text{props}|M^s)\right] + \text{KL}\left[q_{\phi^s_\text{no-props}}(\theta^s_\text{no-props}) || p_\delta(\theta^s_\text{no-props})\right].
\end{align*}

From all of the above it them follows that the ELBO in this scenario is
\begin{align*}
    &\sum_{s=1}^S \mathbb{E}_{q_{\phi^s_\text{props}, \phi^s_\text{no-props}}(\theta^s_\text{props}, \theta^s_\text{no-props})}\left[ \log p_{\theta_\text{shared}}(Y^s|X^s, \theta^s_\text{props}, \theta^s_\text{no-props}) \right]  \\
    &\quad\quad\quad\quad
    -\text{KL}\left[q_{\phi^s_\text{props}}(\theta^s_\text{props}) || p_\gamma(\theta^s_\text{props}|M^s)\right] - \text{KL}\left[q_{\phi^s_\text{no-props}}(\theta^s_\text{no-props}) || p_\delta(\theta^s_\text{no-props})\right],
\end{align*}
where we define $q_{\phi^s_\text{props}}(\theta^s_\text{props})_{ \phi^s_\text{no-props}}(\theta^s_\text{no-props}) := q_{\phi^s_\text{props}}(\theta^s_\text{props})q_{\phi^s_\text{no-props}}(\theta^s_\text{no-props})$.

\section{The CPD can learn to represent variability}
\label{sec:proofs}

We have seen how the CPD encourages the posteriors over model parameters for each domain to be similar, but is it possible for a learned CPD to reflect uncertainty in model parameters arising from variability not predicted by measurable properties? In the case of discrete measurable properties, we can formally prove that this is true. Here we rigorously establish this result as a Theorem for discrete $\theta^s$. 

\begin{lemma}
    Let $q^s(\theta^s)$ for $s=1, \ldots, S$ be a finite set of probability density functions with finite entropy over the continuous random variables $\theta^s \in \mathbb{R}^{d_\theta \times m}$. Then $\sum_{s=1}^S\text{KL}\left[q^s(\theta^s) || p(\theta^s)\right]$ is minimized with respect to $p(\theta^s)$ when $p(\theta^s) = \frac{1}{S} \sum_{s'=1}^S q^{s'}(\theta^s)$.
    \label{lemma:base}
\end{lemma}

\begin{proof}

    We begin by noting that all $\theta^s \in \mathbb{R}^{d_\theta \times m}$.  This means the lemma could be equivalently stated and will hold if we can prove the result for a set of distributions $q^s(\theta)$ and $p(\theta)$ over $\theta \in \mathbb{R}^{d_\theta \times m}$.

    Starting with the definition of the KL-divergence, we derive
    \begin{align*}
        \sum_{s=1}^S\text{KL}\left[q^s(\theta) || p(\theta) \right] &= \sum_{s=1}^S \int q^s(\theta) \left( \log q^s(\theta) - \log p(\theta)\right) d\theta \\
        &= \int \left( \sum_{s=1}^S   q^s(\theta) \left( \log q^s(\theta) - \log p(\theta)\right) \right) d\theta \\
        &= \int \left(  \sum_{s=1}^S   q^s(\theta) \log q^s(\theta)  - \sum_{s=1}^S   q^s(\theta)\log p(\theta)\right) d\theta \\
    &= \underbrace{\int \left( \sum_{s=1}^Sq^s(\theta)\log q^s(\theta) \right) d\theta}_{:=C} - \int \left( \sum_{s=1}^Sq^s(\theta)\log p(\theta) \right) d\theta.\\
    \end{align*}

    Note that $\int \left( \sum_{s=1}^Sq^s(\theta)\log q^s(\theta) \right) d\theta =   \sum_{s=1}^S \int q^s(\theta)\log q^s(\theta)  d\theta = -\sum_{s=1}^S \text{H}[q^s(\theta)]$, which will be finite, since the entropy of each $q^s(\theta)$ is finite by assumption. Defining $C:=-\sum_{s=1}^S \text{H}[q^s]$, we continue the proof by writing 
    \begin{align*}
        \sum_{s=1}^S\text{KL}\left[q^s(\theta^s) || p(\theta^s) \right] &= C - \int \left( \sum_{s=1}^Sq^s(\theta)\log p(\theta) \right) d\theta \\
        &= C - S \int \left( \frac{1}{S}\sum_{s=1}^Sq^s(\theta) \log p(\theta) \right) d\theta.
    \end{align*}
    We now define $q^*(\theta) := \frac{1}{S}\sum_{s'=1}^S q^{s'}(\theta)$.  It can be easily verified that $q^*(\theta)$ is a probability density function.  We can then write
    \begin{align}
        \sum_{s=1}^S\text{KL}\left[q^s(\theta) || p(\theta) \right] = C - S\int q^*(\theta)\log p(\theta)d\theta = C + S\mathbb{H}[q^*(\theta), p(\theta)], \label{eq:proof_cross_entropy}
    \end{align}
    where $\mathbb{H}[q^*(\theta), p(\theta)]$ is the cross entropy of $p(\theta)$ relative to $q^*(\theta)$.  To minimize \Eq{proof_cross_entropy}, we must minimize $\mathbb{H}[q^*(\theta), p(\theta)]$, which will occur when  $p(\theta) = q^*(\theta)$.
\end{proof}

\begin{theorem}
    Let $q_{\phi^s}(\theta^s)$ for $s \in 1, \ldots, S$ be a finite set of be probability density functions with finite entropy for the continuous random variables $\theta^s \in \mathbb{R}^{d_\theta \times m}$. Further, let $M^s$ take on values from some finite set, $\mathcal{M}$, for all $s$.  Finally, let $n(\lowerscript{m})$ be the number of  $s$ such that $M^s = \lowerscript{m}$.   Then $\sum_{s=1}^S\text{KL}\left[q_{\phi^s}(\theta^s)||p(\theta^s|M^s) \right]$ is minimized when $p(\theta^s|M^s = \lowerscript{m}) = \frac{1}{n(\lowerscript{m})}\sum_{s': M^{s'} = \lowerscript{m}} q_{\phi^{s'}}(\theta^{s})$ for all $\lowerscript{m}$ such that $n(\lowerscript{m}) > 0$.
    \label{theorem:main}
\end{theorem}

\begin{proof}
    Note that we can write 
    \begin{align*}
    \sum_{s=1}^S\text{KL}\left[q_{\phi^s}(\theta^s)||p(\theta^s|M^s)  \right] = \sum_{\lowerscript{m}:  n(\lowerscript{m}) > 0}\sum_{s:M^s=\lowerscript{m}} \text{KL}\left[q_{\phi^s}(\theta^s)||p(\theta^s|M^s = \lowerscript{m})  \right]
    \end{align*}
    From \Lma{base}, each term $\sum_{s:M^s=\lowerscript{m}} \text{KL}\left[q_{\phi^s}(\theta^s)||p(\theta^s|M^s = \lowerscript{m})  \right]$ will be minimized when $p(\theta^s|M^s = \lowerscript{m}) = \frac{1}{n(\lowerscript{m})}\sum_{s': M^{s'} = \lowerscript{m}} q_{\phi^{s'}}(\theta^s)$.
\end{proof}

\Thm{main} can be understood as a generalization of previous work that established similar results where the distribution to optimize was not conditional \cite{hoffman2016, tomczak2018}. While \Thm{main} applies directly when $\theta^s$ is the same dimensionality across domains, it can be immediately applied to derive specialized results, such as the following corollary that applies when $\theta^s$ are continuous random variables and the CPD and approximate posteriors factor as described in \Sect{cpd}. 

\begin{corollary}
    Let $q_{\phi^s}(\theta^s)$ for $s \in 1, \ldots, S$ be a finite set of probability density functions for the random variables $\theta^s \in \mathbb{R}^{d^s_\theta \times m}$ with row length $m$.  Assume that for each $s$, $q_{\phi^s}(\theta^s) = \prod_{i=1}^{d^s_\theta} q_{\phi^s_i}(\theta^s[i,:])$ for some continuous probability density function $q_{\phi^s_i}(\theta^s[i, :])$ with finite entropy. Associate with each $s$ some $M^s \in \mathbb{R}^{d^s_\theta \times r}$, and assume the rows of $M^s$ can take values from some finite set $\mathcal{M}$. 
    
    Define $p(\theta^s|M^s) = \prod_{i=1}^{d^s_\theta} p(\theta^s[i,:]|M^s[i,:])$, where $p(\theta^s[i,:]|M^s[i,:] = \lowerscript{m})$ is a single continuous probability density function for each $\lowerscript{m} \in \mathcal{M}$. Then $\sum_{s=1}^S \text{KL}(q_{\phi^s}(\theta^s) || p(\theta^s|M^s))$ will be minimized when $p(\theta^s[i,:]|M^s[i,:] = \lowerscript{m}) = \frac{1}{n(\lowerscript{m})}\sum_{s',i: M^{s'}[i,:] = \lowerscript{m}} q_{\phi^{s'}_i}(\theta^{s}[i,:])$ for all $\lowerscript{m}$ such that $n(\lowerscript{m}) > 0$, where $n(\lowerscript{m}) = \sum_{s=1}^S \sum_{i=1}^{d^s_\theta} \mathbb{I}(M^s[i,:] = \lowerscript{m})$, and $\mathbb{I}(\cdot)$ is the indicator function. 
    \label{corollary:main}
\end{corollary}

\begin{proof}

First note that 
    \begin{align*}
        &\text{KL}(q_{\phi^s}(\theta^s) || p(\theta^s|M^s)) = \int q_{\phi^s}(\theta^s) \left( \log q_{\phi^s}(\theta^s) - \log p(\theta^s|M^s) \right) d\theta^s \\
        &\quad = \int  \prod_{j=1}^{d^s_\theta} q_{\phi^s_j}(\theta^s[j,:]) \left(\log \prod_{i=1}^{d^s_\theta} q_{\phi^s_i}(\theta^s[i,:]) - \log \prod_{i=1}^{d^s_\theta} p(\theta^s[i,:]|M^s[i,:]) \right) d\theta^s \\
        &\quad = \int  \prod_{j=1}^{d^s_\theta} q_{\phi^s_j}(\theta^s[j,:]) \left(\ \sum_{i=1}^{d^s_\theta} \log q_{\phi^s_i}(\theta^s[i,:]) -  \sum_{i=1}^{d^s_\theta} \log p(\theta^s[i,:]|M^s[i,:]) \right) d\theta^s \\
        &\quad = \int \sum_{i=1}^{d^s_\theta}  \prod_{j=1}^{d^s_\theta}  q_{\phi^s_j}(\theta^s[j,:]) \left( \log q_{\phi^s_i}(\theta^s[i,:]) - \log p(\theta^s[i,:]|M^s[i,:]) \right)d\theta^s\\
        &\quad =  \sum_{i=1}^{d^s_\theta} \int \prod_{j=1}^{d^s_\theta}  q_{\phi^s_j}(\theta^s[j,:]) \left( \log q_{\phi^s_i}(\theta^s[i,:]) - \log p(\theta^s[i,:]|M^s[i,:]) \right)d\theta^s\\
        &\quad =  \sum_{i=1}^{d^s_\theta} \int q_{\phi^s_i}(\theta^s[i,:])\left(\prod_{j\neq i}^{d^s_\theta}   q_{\phi^s_j}(\theta^s[j,:])\right) \left( \log q_{\phi^s_i}(\theta^s[i,:]) - \log p(\theta^s[i,:]|M^s[i,:]) \right)d\theta^s\\
        &\quad = \sum_{i=1}^{d^s_\theta} \int q_{\phi^s_i}(\theta^s[i,:]) \left( \log q_{\phi^s_i}(\theta^s[i,:]) - \log p(\theta^s[i,:]|M^s[i,:])  \right) d\theta^s[i,:] \\
        &\quad = \sum_{i=1}^{d^s_\theta} \text{KL}(q_{\phi^s_i}(\theta^s[i,:]) || p(\theta^s[i,:]|M^s[i,:]))
    \end{align*}
    From here it follows that 
    \begin{align*}
        \sum_{s=1}^S \text{KL}(q^s(\theta^s) || p(\theta^s|M^s)) = \sum_{s=1}^S \sum_{i=1}^{d^s_\theta} \text{KL}(q_{\phi^s_i}(\theta^s[i,:]) || p(\theta^s[i,:]|M^s[i,:]))
    \end{align*}
    Since the objective decomposes as a sum of row-wise KL divergences, we can apply \Thm{main} to the collection $\{q_{\phi^s_i}\}_{s,i}$ with labels $\{M^s[i,:]\}_{s,i}$. This yields the stated mixture solution for each value of $m$.
\end{proof}

\section{Supplemental methods}

\subsection{Additional methods for the simulated example}
\label{sec:simulated_ex_details}

\paragraph{The form of the ground-truth CPD}

We specify the mathematical form of $\mu$ and $\sigma$ for the ground-truth CPD, visualized in \Fig{syn_results}b, for the simulated example. These were randomly generated functions mapping from $[0, 1]\times[0,1]$ to $\mathbb{R}$ of the following form
\begin{align*}
    \mu(m) &= \sum_{i=1}^{50} g_i(m) \\
    \sigma(m) &= \sum_{i=1}^{50} |h_i(m)| + .01,
\end{align*}
where $|\cdot|$ denotes absolute value and $g_i$ and $h_i$ were randomly generated Gaussian bump functions of the general form $a e^{-\frac{||m - c||_2^2}{0.2^2}}$, where $a \in \mathbb{R}$ and $c \in \mathbb{R}^2$ were randomly sampled magnitudes and centers of each bump function.  The centers were sampled i.i.d. from a Uniform distribution over the unit square.  When generating $\mu$, magnitudes were sampled i.i.d. from a $\mathcal{N}(0, 1)$ distribution and when generating $\sigma$ magnitudes were sampled i.i.d from a $\mathcal{N}(0, .1)$ distribution.

\paragraph{Data generation}

We provide specifics of how we generated training data in the simulated example.   Neurons in half of the property space were fist selected to be silent in the training data for each individual. We did this by selecting neurons in an arbitrarily chosen half of property space to be silent for individual one (\Fig{syn_results}, upper plot) and then rotating the region of silent neurons by $90$ degrees clockwise for each subsequently generated individual.  

For each individual, we also simulated the activity of the non-silent neurons so that values of $l^s_t = \omega^s x_t^s$ fell within pseudo-randomly selected intervals of length 1 within the domain, $[-2, 2)$, of the shared function $f$, as shown for one randomly selected interval in the gray region of the plot of $f$ in \Fig{syn_results}a. These intervals were selected as follows.  For the first four individuals these intervals were set to $[-2, -1), [-1, 0), [0, 1), $ and  $[1, 2)$ to ensure that across the collection of data generated for all individuals the entire domain of $f$ was explored.  Intervals for subsequently generated individuals were selected by randomly selecting a leading edge from a $\text{Uniform}[-2, 1]$ distribution and then assigning the trailing edge so the length of each interval was 1. Having selected the interval that values of $l_s^t$ should fall in for any individual, values of $x_s^t$ were then generated as follows.  First, we selected target values of $l^s_t$ for each time point i.i.d. from a uniform distribution over the interval for each individual.  For each target value, we then formed a vector $x^s_{t, \text{base}} \in \mathbb{R}^{d^s_x}$ that 1) had zeros for all silent neurons, and such that 2) the portion of $x^s_{t, \text{base}}$ corresponding to the non-silent neurons was in the same direction as the portion of $\omega^s$ for the non-silent neurons and 3) $(\omega^s)^T x^s_{t, \text{base}}$ was equal to the target value.  We then randomly generated a vector, $x^s_{t, \text{noise}}$ by 1) sampling entries for the non-silent neurons i.i.d. from a $\mathcal{N}(0, 1)$ distribution and then projecting that generated vector onto the subspace orthogonal to $x^s_{t, \text{base}}$. We then formed $x^s_t = x^s_{t, \text{base}} + x^s_{t, \text{noise}}$.  

\paragraph{The form of fit CPD}

We provide the mathematical description of the form of the CPD we fit in the simulated example.  As specified in \Sect{syn_mdl_fitting}, we selected a CPD that factorized according to 
\begin{align*}
    \hat{p}_\gamma(\omega^s| M^s) &= \prod_{i=1}^{d^s_x} \mathcal{N}(\hat{\mu}_{\gamma_\mu}(M[i,:]), \hat{\sigma}_{\gamma_\sigma}(M[i,:])),
\end{align*}
where $\hat{\mu}_{\gamma_\mu}$ and $\hat{\sigma}_{\gamma_\sigma}$ were learned functions.  The function $\hat{\mu}_{\gamma_\mu}$ was a SHBF function, as described above in \Sect{SHBF_fcns}, where the hyperrectangles defining the basis functions were arranged to cover the unit square in a $100 \times 100$ pattern with hyperrectangles overlapping each other by fifty percent in each dimension.  The learned parameters, $\gamma_\mu$, were the coefficients for each hyperrectangle of the SHBF function. The function $\hat{\sigma}_{\gamma_\sigma}$ was defined as $\sigma_{\gamma_\sigma}(m) = e^{g_{\gamma_\sigma}(m)} + 10^{-6}$, where $g_{\gamma_\sigma}$ was another SHBF function with the hyperrectangles defining its basis functions arranged identically to those of the SHBF function for $\hat{\mu}_{\gamma_\mu}$.  The parameters $\gamma_\sigma$ were again the coefficients for each hyperrectangle of the SHBF function, $g_{\gamma_\sigma}$.

\paragraph{The form of the fit $p_\delta(\nu^s)$}

In the simulated example, we fit a prior over the noise standard deviation of the form $p_\delta(\nu^s) = \Gamma(g_\alpha(\delta_\alpha), g_\beta(\delta_\beta))$ where $\delta = \{\delta_\alpha, \delta_\beta \}$ were learnable parameters and $g_\alpha$ and $g_\beta$ were tanh function scaled and shifted so the shape parameter was bounded between $1$ and $10^{3}$ and the rate parameter was bounded between $10^{-1}$ and $10^4$.

\paragraph{The form of $\hat{f}_{\theta_\text{shared}}$}

The form of $\hat{f}_{\theta_\text{shared}}: \mathbb{R} \rightarrow \mathbb{R}$ we fit for the simulated example was a neural network specified as follows.  First, the input was scaled by $.001$ to improve numerical stability.  When comparing the fit $\hat{f}_{\theta_\text{shared}}$ to the true $f$, this scaling was accounted for (see section below on identifiability). The scaled input was then processed through a neural network with 2 hidden layers. We define the input of the neural network as $h_{0} \in \mathbb{R}$. The activity of the $i^{th}$ hidden layer, $h_i \in \mathbb{R}^{d_i}$, is defined recursively as
\begin{align*}
    h_i = [h_{i-1}, \text{ReLU}(W_{i-1} h_{i-1} + b_{i-1})],
\end{align*}
\noindent where $[\cdot,\cdot]$ represents concatenation, $\text{ReLU}$ is the rectified linear transfer function, $W_{i-1} \in \mathbb{R}^{d_{i-1}+r_g \times d_{i-1}}$ is a weight matrix with growth rate $r_g=10$, and  $b_{i-1} \in \mathbb{R}^2$ is a bias vector.  Output of the neural network is calculated from $h_{2} \in \mathbb{R}^{21}$ as $w_{2} h_{2} + b_{2}$, for a weight matrix $w_{2} \in \mathbb{R}^{1 \times 21}$ and bias $b_{2} \in \mathbb{R}$.  The parameters $\theta_\text{shared}$ are the weights and biases between all layers.  The concatenation of the input to the output of each hidden layer was inspired by DenseNets \cite{huang2017}, and we found in practice it aided model fitting, likely by addressing the problem of vanishing gradients that can occur with deep neural networks \cite{bengio1994}.

\paragraph{The form of the approximate posteriors}

In the simulated example, we used a Gaussian mean field approximation 
\begin{align*}
q^s_{\phi^s_\text{props}}(\omega^s) = \prod_{i=1}^{d_x^s}\mathcal{N}(\phi^s_{\text{props}, \mu, i}, g_\sigma(\phi^s_{\text{props}, \sigma, i})),
\end{align*}
for the approximate posterior over $\omega^s$ for each individual, where $\phi^s_{\text{props}, \mu, i}$ and $\phi^s_{\text{props}, \sigma, i}$ are learnable parameters determining the mean and standard deviation for $\omega[i]$.  We pass $\phi^s_{\text{props}, \sigma, i}$ through $g_\sigma$, a $\text{tanh}$ functions scaled and shifted so that standard deviations are bounded between $10^{-6}$ and $10$.

For each system instance $q_{\phi_\text{no-props}^s}(\nu^s)$ was a $\Gamma(g_{\alpha}(\phi^s_{\text{no-props}, \alpha}), g_{\beta}(\phi^s_{\text{no-props}, \beta}))$ distribution, where
$\phi_\text{no-props}^s = \{\phi^s_{\text{no-props}, \alpha}, \phi^s_{\text{no-props}, \beta} \}$
were learnable parameters determining the shape and rate parameters of the distribution, and $g_{\alpha}$ and $g_{\beta}$ were defined as above.

\paragraph{Additional fitting details}

We now provide the details of how synthesis was performed for the simulated example.  When fitting models to one individual at a time, fitting was identical, with the only difference being data for only one individual was used in the description below.  Following the strategy described in \Sect{DPMS}, we performed synthesis by using stochastic gradient ascent to maximize a sampled approximation of the ELBO.  In each iteration we first sampled 
\begin{align*}
    \omega^s_\text{sample} &\sim q_{\phi^s_\text{props}}(\omega^s) \\
    \nu^s_\text{sample} &\sim q_{\phi^s_\text{no-props}}(\nu^s),
\end{align*}
using the reparameterization trick \cite{kingma2014} and then took a gradient step to optimize the following objective
\begin{align}
 \sum_{s=1}^S  c^s\log p_{\theta_\text{shared}}(Y^s_\text{mb}|X^s_\text{mb}, \omega^s_\text{sample}, \nu^s_\text{sample})  - \text{KL}\left[q_{\phi^s_\text{props}}(\omega^s) || p_\gamma(\omega^s|M^s)\right]  
 - \text{KL}\left[q_{\phi^s_\text{no-props}}(\nu^s) || p_\delta(\nu^s)\right], \label{eq:elbo_simulated_ex}
\end{align}
where the derivatives for both KL terms were computed analytically.   Here $Y^s_\text{mb}$ and $X^s_\text{mb}$ are pseudorandomly selected mini-batches of data, sized so that all samples for an individual were processed every two iterations. The variables $c^s$ is the ratio of the total number of samples in the data for an individual divided by the number of samples in a mini-batch.  

We used constrained posterior initialization, as described in \Sect{initialization}, to initialize the values of  $\theta_\text{shared}$, $\gamma$, $\delta$, $\phi^s_\text{props}$ and $\phi^s_\text{no-props}$, while fixing the parameters of $\gamma$ determining variance, so that the CPD predicted a constant standard deviation of $.01$ for all properties.  Synthesis was then performed starting with the parameter values produced by the constrained posterior initialization. Constrained posterior initialization and synthesis were performed via stochastic gradient ascent with the Adam optimizer \cite{kingma2015}.   Gradient ascent for constrained posterior initialization was performed for 500 epochs with a fixed learning rate of $.01$ and values for the decay rates of the moment estimates for the Adam optimizer of $\beta_1 = .9$ and $\beta_2 = .999$ throughout.  Synthesis was performed for 3000 epochs, with decay rates of $\beta_1 = .9$ and $\beta_2 = .999$ throughout and starting with a learning rate of $.1$ that was decreased by a factor of $10$ every $1000$ epochs.
The initialization methods for parameters before starting constrained parameter initialization are listed in \Table{simulated_cpi_init_methods}.  Finally, a retrospective form of early stopping was performed by saving checkpoints every 100 epochs and then using the checkpoint with the best model performance on validation data.  Validation data, consisting of 1000 time points for each individual, was generated identically to the training data for each individual.  Model performance for early stopping was based on R-squared between true and predicted swim signals.  When predicting swim signals, the posterior means of the approximate posteriors were used as model parameters when predicting output. For synthesized models, the best checkpoint was determined by averaging the R-squared values for all individuals models were for synthesized for.  For models fit to individuals in isolation, the best checkpoint was selected on the R-squared for that individual alone.   

\begin{table}[h!]
    \begin{center}
        \begin{tabular}{ |c|c|} 
        \hline
         $\gamma$ & \thead{All entries of $\gamma_\mu$ set to 0 \\
         All entries of $\gamma_\sigma$ set so that the CPD predicted a standard deviation of .01 for all properties}\\
         \hline
         $\delta$ &\thead{$\delta_\alpha$ set so the shape parameter was 10 \\
         $\delta_\beta$ set so the rate parameter was 1000} \\
        \hline
        $\theta_\text{shared}$ & \thead{Entries of $W_i$ and $b_i$ initialized from a $\text{Uniform}[-\sqrt{1/d_i}, \sqrt{1/d_i}]$ distribution  } \\ 
        \hline
        $\phi^s_\text{props}$ &  \thead{Entries of $\phi^s_{\text{props}, \mu}$ initialized from a $\mathcal{N}(0, .01)$ distribution  \\
        Entries of $\phi^s_{\text{props}, \sigma}$ set so posterior standard deviations were $.01$ for all neurons} \\ 
        \hline
        $\phi^s_\text{no-props}$ & \thead{$\phi^s_{\text{no-props}, \alpha}$ set so the shape parameter was $10$ \\
        $\phi^s_{\text{no-props}, \beta}$ set so the rate parameter was $1$} \\ 
        \hline
        \end{tabular}
    \end{center}
    \caption{Methods of initializing parameter values before starting the constrained posterior initialization.  See subsections above for definition of broken out parameters in the right column.  The parameters $\phi^s_\text{props}$ are not listed because the approximate posteriors, $q_{\phi^s_\text{props}}(w^s)$, are constrained to be equal to the CPD during constrained posterior initialization.}
    \label{tbl:simulated_cpi_init_methods}
\end{table}

\paragraph{Accounting for non-identifiability when comparing synthesized models to ground truth}

The form of the models we fit in the simulated example have a fundamental non-identifiability because it is possible that for any model with projection weights $\omega^s$ and shared function $f$ to define a pair $(\omega^s)' = k\omega^s$ and $f'(l^s_t) = f\left(\frac{l^s_t}{k}\right)$ that will define the same input-output relationship.  This means the absolute scale of the weights we model with the CPD and approximate posteriors and the scale for the domain of $f$ cannot be learned from data.  Therefore, when visualizing $\hat{f}$, the estimated mean and standard deviation functions of the CPD as well as the mean of the approximate posteriors in \Fig{syn_results}, we estimated a $k$ to account for this and visualized the estimated entities with this learned scaling factor applied.  The value of $k$ was chosen as the scaling factor that minimized the squared error between the estimated and true means of the CPD, evaluated over a grid of points. When calculating this factor for CPDs learned from individuals in isolation, we compared only the portion of the true and estimated CPDs over the half of property space for the non-silent neurons. 

\paragraph{Calculating the ELBO on test data.}

We calculated the ELBO for the test data for an individual in the same way we calculated the ELBO for each value of $s$ in the sum above in \Eq{elbo_simulated_ex} with the only differences being that 1) $Y^s_\text{mb}$ and $X^s_\text{mb}$ were set to the entirety of the test data, 2) $c^s$ was set to 1, and 3)
to improve the accuracy of the approximated ELBO for model evaluation we used $1000$ samples from the approximate posteriors when approximating the expected log-likelihood for a model.

\subsection{Additional methods for the synthesis of regression models with neural data}
\label{sec:reg_real_data_details}

\paragraph{Data selection and partitioning into train, validation and test sets}

For this analysis, we used phototaxis data from 7 fish, corresponding to fish 1,2,5,6 (base fish), and 8, 10, and 11 (target fish) from \cite{chen2018}. These fish were selected because they were imaged as similar rates and were judged to have robust swimming responses during phototaxis. We note that the data made publicly available in \cite{chen2018} includes periods of phototaxis where shock was also delivered, and we omitted any periods with shock from our analysis. 

We divided the phototaxis data for a fish into train, validation and test sets as follows.  To roughly balance the swim vigor represented in each set, we partitioned the data for each fish into sequential \emph{chunks} five time points long and calculated the maximum at each point in time of the average of the swim signals on both channels. We then ordered the chunks by maximum swim vigor and iteratively assigned them to 42 different \emph{groups}, randomly assigning one of each of the first 42 chunks to each group and then repeating until there were no longer enough chunks left to assign to all 42 groups. For each fish, this produced 42 groupings exactly balanced in data quantity and roughly balanced in swim vigor.   To assign data to folds for three-fold cross validation, we assigned 14 groups each to train, validation and test sets, rotating which groups were assigned to each across folds. 

When performing synthesis with less than the full amount of test and validation data we simply used a smaller number of the groups assigned to the test and validation sets for a given fold.

\paragraph{The form of the fit CPD}

As specified in \sect{app_details_reg}, we used a CPD that factorized according to
\begin{align*}
    p_\gamma(\omega^s| M^s) = \prod_{i=1}^{d_x^s}\prod_{j=1}^{10} \mathcal{N}(\mu_{\gamma_{1,j}}(M[i,:]), \sigma_{\gamma_{2,j}}(M[i,:])),
\end{align*}
where $\mu_{\gamma_{1,j}}$ and $\sigma_{\gamma_{2,j}}$ were learned functions for each column of $\omega^s$.  These functions were identical in form to those described in \Sect{simulated_ex_details} except the hyperrectangles defining the basis functions of the SHBF functions were laid out differently. Because properties were the 3D position of neurons in the brain, we used a grid of 3D non-overlapping rectangles covering the brain of each individual.  The grid had 140 partitions along the anterior-posterior axis, 50 partitions along the left-right axis and 20 partitions along the dorsal-ventral axis. The size of grid and number of partitions per dimension was selected to produce enough 3D rectangles of roughly cubic shape to learn the CPD at fine resolution. 

\paragraph{The form of the fit $p_\delta(\nu^s)$}

We learned Gamma priors of the same form as in the simulated example over the noise standard deviations, $\nu^s \in \mathbb{R}_+^2$, for the swim signals.  Separate priors were learned for the left and right channels so $\delta = \{\delta_\alpha \in \mathbb{R}^2, \delta_\beta \in \mathbb{R}^2$ \}, where the separate entries of $\delta_\alpha$ and $\delta_\beta$ determining the shape and rate parameters of the Gamma distribution for the left and right channels.

\paragraph{The form of $\hat{f}_{\theta_\text{shared}}$}

Aside for differences in scaling applied to input, which was $.01$, and differences in input and output dimensionality, the form of $\hat{f}_{\theta_\text{shared}}: \mathbb{R}^{10} \rightarrow \mathbb{R}^2$ used in \Sect{syn_example} was the same as that described in \Sect{simulated_ex_details}. That is input was scaled by .01 and then passed through a network with
 2 hidden layers that concatenated their output to their input and again a growth rate of $r_g=10$, each followed by the application of a ReLU nonlinearity and a final linear projection down to two dimensions was applied. 

\paragraph{The form of the approximate posteriors}

We used a mean-field approximation for $q^s_{\phi_\text{props}}(\omega^s)$ for each system instance of the form
\begin{align*}
    q^s_{\phi_\text{props}}(\omega^s) = \prod_{j=1}^{10}\prod_{i=1}^{d_x^s} \mathcal{N}(\phi^s_{\text{props}, \mu, j, i}, g_\sigma(\phi^s_{\text{props}, \sigma, j, i})),
\end{align*}
where $\phi^s_{\text{props}, \mu, j, i}$ and $\phi^s_{\text{props}, \sigma, j, i}$ are learnable parameters determining the mean and standard deviation for $\omega[j,i]$, and $g_\sigma$ is defined as above for the simulated example. 

For each system instance $q_{\phi_\text{no-props}^s}(\nu^s)$ was defined as the product of two Gamma distributions, one for each channel, and each of the same form as the approximate posteriors over the noise standard deviation in the simulated example. Therefore $\phi^s_{\text{no-props}} = \{\phi^s_{\text{no-props}, \alpha} \in \mathbb{R}^2, \phi^s_{\text{no-props}, \beta} \in \mathbb{R}^2 \}$, where the separate entries of $\phi^s_{\text{no-props},  \alpha} $ and $\phi^s_{\text{no-props}, \beta}$ determine the shape and rate parameters of the distributions for each channel. 

\begin{table}[h]
    \begin{center}
        \begin{tabular}{ |c|c|} 
        \hline
         $\gamma$ & \thead{All entries of $\gamma_\mu$ set to 0 \\
         All entries of $\gamma_\sigma$ set so that the CPD predicted a standard deviation of .01 for all properties and dimensions $j$}\\
         \hline
         $\delta$ &\thead{Entries of $\delta_\alpha$ set so the shape parameter for each channel was 10 \\
         Entries of $\delta_\beta$ set so the rate parameter for each channel was 10} \\
        \hline
        $\theta_\text{shared}$ & \thead{Entries of $W_i$ and $b_i$ initialized from a $\text{Uniform}[-\sqrt{1/d_i}, \sqrt{1/d_i}]$ distribution  } \\
        \hline
        $\phi^s_\text{props}$ &  \thead{Entries of $\phi^s_{\text{props}, \mu, j, i}$ initialized from a $\mathcal{N}(0, .01)$ distribution for all neurons $i$ and dimensions $j$ \\
        Entries of $\phi^s_{\text{props}, \sigma, j, i}$ set so posterior standard deviations were $.01$ for all neurons $i$ and dimensions $j$} \\ 
        \hline
         $\phi^s_\text{no-props}$ & \thead{Entries of $\phi^s_{\text{no-props}, \alpha}$ set so the shape parameter for each channel was $10$  \\
        Entries of $\phi^s_{\text{no-props}, \beta}$ set so the rate parameter for each channel was $10$} \\ 
        \hline
        \end{tabular}
    \end{center}
    \caption{Methods of initializing parameter values before starting the constrained posterior initialization.  See above for definition of broken out parameters in the right column.  The parameters $\phi^s_\text{props}$ are not listed because the approximate posteriors, $q_{\phi^s_\text{props}}(w^s)$, are constrained to be equal to the CPD during constrained posterior initialization.}
    \label{tbl:real_reg_cpi_init_methods}
\end{table}

\paragraph{Additional fitting details}

Fitting, both when performing synthesis and when fitting to data for fish in isolation was performed in the same manner as described for the simulated example above with the exception of differences in initial values for starting the constrained posterior initialization (see \Table{real_reg_cpi_init_methods}) and learning rate schedules.  Constrained posterior initialization was performed for 500 epochs with a learning rate of $1\times 10^{-4}$, followed by 500 epochs with a learning rate of $1\times10^{-5}$. Synthesis was then performed for 10000 epochs with a learning rate of $1\times10^{-5}$, followed by 10000 more epochs with a learning rate of $1\times10^{-6}$.  During synthesis, checkpoints were saved every 500 epochs, and retroactive early-stopping was applied using validation data to prevent overfitting. 
This early stopping was performed exactly described in the fitting details for the simulated example, differing only in that we used the average R-squared across both outputs when measuring the performance of a model for an individual because that model predictions were now two-dimensional. 

\subsection{Additional methods for the synthesis of factor analysis models with neural data}
\label{sec:add_fa_methods}

\paragraph{Data selection}

For these analyses, we used data from 3 fish corresponding to fish 8, 9 and 11 from the original dataset \cite{chen2018}.  These fish were selected as they were imaged at similar rates and were judged to display robust OMR swimming responses. We omitted any data for the three included fish during periods in which shock was administered. We note that fish 10 in the original dataset also was imaged at a similar rate to those just mentioned but was recorded as turning in opposite directions from what would be expected to left and right OMR.  Upon further analysis, we believe this due to a simple mislabeling of directions in the released data. However, since the analysis was performed most naturally with just three fish (since we were examining three behaviors so that one behavior could be cleanly assigned to each fish), we still chose to omit fish 10 from this analysis.

\paragraph{The form of the fit CPD}

As discussed in \Sect{dr_app_details}, the CPD was of the form
\begin{align*}
    p_\gamma(\Lambda^s, \eta^s, \nu^s | M^s) = p_{\gamma_\Lambda}(\Lambda^s|M^s)p_{\gamma_\eta}(\eta^s|M^s)p_{\gamma_\nu}(\nu^s| M^s)
 \end{align*}
for the learnable parameters $\gamma = \{\gamma_\Lambda, \gamma_\eta, \gamma_\nu \}$ where $p_{\gamma_\Lambda}(\Lambda^s|M^s)$ was specified as 
\begin{align*}
    p_{\gamma_\Lambda}(\Lambda^s|M^s) = \prod_{i=1}^{d_x^s}\prod_{j=1}^{10} \mathcal{N}(\mu_{\gamma_{\Lambda, \mu, j}}(M[i,:]), \sigma_{\gamma_{\Lambda, \sigma, j}}(M[i,:])) 
\end{align*}
for the parameters $\gamma_\Lambda = \{\gamma_{\Lambda, \mu, j}, \gamma_{\Lambda, \sigma, j} \}_{j=1}^{10}$ and $\mu_{\gamma_{\Lambda,\mu,j}}$ and $\sigma_{\gamma_{\Lambda,\sigma,j}}$ are functions of the exact same form as those of the CPD described in \Sect{reg_real_data_details}. That is $\mu_{\gamma_{\Lambda,\mu,j}}$ was an SHBF function and $\sigma_{\gamma_{\Lambda,\sigma,j}}$ was a transformed version of an SHBF function with the hyperrectangular basis functions underlying the SHBF functions laid out on a $140\times50\times20$ non-overlapping grid. The component $p_{\gamma_\eta}(\eta^s|M^s)$ of the CPD was specified as 
\begin{align*}
    p_{\gamma_\eta}(\eta^s|M^s) = \prod_{i=1}^{d^s_x}\mathcal{N}(\mu_{\gamma_{\eta,\mu}}(M[i,:]), \sigma_{\gamma_{\eta,\sigma}}(M[i,:]))
\end{align*}
where $\gamma_\eta = \{\gamma_{\eta,\mu}, \gamma_{\eta,\sigma} \}$ are learnable parameters of functions that are again of the same form as those of the CPD described in \Sect{reg_real_data_details}.  Finally, $p_{\gamma_\nu}(\nu^s| M^s)$ was specified as 
\begin{align*}
    p_{\gamma_\nu}(\nu^s| M^s) = \prod_{i=1}^{d^s_x}\Gamma(g_{\alpha_{\gamma_{\nu,\alpha}}}(M[i,:])), g_{\beta_{\gamma_{\nu,\beta}}}(M[i,:])))
\end{align*}
where $\alpha_{\gamma_{\nu,\alpha}}$ and $\beta_{\gamma_{\nu,\beta}}$ were SHBF functions with hyperrectangular basis functions arranged in the same manner as those for the SHBF functions underlying the components of the CPD for $\xi$ and $\mu$ and $g_{\alpha_{\gamma_{\nu,\alpha}}}$ and $g_{\beta_{\gamma_{\nu,\beta}}}$ were scaled and shifted tanh functions as described above in \Sect{simulated_ex_details}.

\paragraph{The form of the approximate posteriors}

We now provide the mathematical form of $q_{\phi^s_\theta}(\xi^s, \mu^s, \sigma^s)$ and $q_{\phi_z^s}(Z^s)$.  We specified 
\begin{align*}
    q_{\phi^s_\theta}(\xi^s, \mu^s, \sigma^s) = q_{\phi^s_{\theta, \xi}}(\xi^s)q_{\phi^s_{\theta, \mu}}(\mu^s)q_{\phi^s_{\theta, \sigma}}(\sigma^s),
\end{align*}
for the parameters $\phi^s_\theta = \{\phi^s_{\theta, \xi}, \phi^s_{\theta, \mu}, \phi^s_{\theta, \sigma} \}$,
where 
\begin{align*}
    q^s_{\phi_{\theta, \xi}}(\xi^s) = \prod_{j=1}^{10}\prod_{i=1}^{d_x^s} \mathcal{N}(\phi^s_{\theta, \xi, \mu, j, i}, g_\sigma(\phi^s_{\theta, \xi, \sigma, j, i})),
\end{align*}
where $\phi^s_{\theta, \xi, \mu, j, i}$ and $\phi^s_{\theta, \xi, \sigma, j, i}$ are learnable parameters determining the mean and standard deviation for $\xi[j,i]$, and $g_\sigma$ is defined as above for the simulated example. We defined 
\begin{align*}
    q^s_{\phi_{\theta, \mu}}(\mu^s) = \prod_{i=1}^{d_x^s} \mathcal{N}(\phi^s_{\theta, \mu, \mu, i}, g_\sigma(\phi^s_{\theta, \mu, \sigma, i})),
\end{align*}
where $\phi^s_{\theta, \mu, i}$ and $\phi^s_{\theta, \sigma, i}$ are again learnable parameters determining the mean and standard deviation for $\mu[i]$, and $g_\sigma$ is defined as above for the simulated example.  We specified 
\begin{align*}
    q_{\phi^s_{\theta, \sigma}}(\sigma^s) = \prod_{i=1}^{d_x^s} \Gamma(g_\alpha(\phi^s_{\theta, \sigma, \alpha, i}), g_\beta(\phi^s_{\theta, \sigma, \beta, i})),
\end{align*}
for the functions $g_\alpha$ and $g_\beta$ defined as above in \Sect{simulated_ex_details}, so the parameters $\phi^s_{\theta, \sigma, \alpha, i}$ and $\phi^s_{\theta, \sigma, \beta, i}$ determine the shape and rate parameters of the Gamma distribution over the noise standard deviation for neuron $i$.

Finally, we specified 
\begin{align*}
    q_{\phi_z^s}(Z^s) = \prod_{i=1}^{n^s}\mathcal{N}(m_i^s, \Sigma^s),
\end{align*}
where we defined $\Sigma^s = H^s  (H^s)^T$, for $H^s \in \mathbb{R}^{10 \times 10}$, so the optimized parameters were $\phi_z^s = \{ \{m^s_i\}_{i=1}^{n^s}, H^s \}$. We optimized $H^s$ instead of $\Sigma$ directly because this allowed us to perform unconstrained optimization on $H^s$ while ensuring $\Sigma^s = H^s  (H^s)^T$ was symmetric positive semidefinite. 

\paragraph{Additional fitting details}

We now provide details of how synthesis was performed.  Synthesis was performed with stochastic gradient ascent.  In each iteration, optimization was performed on pseudo-randomly selected mini-batches of data, sized so that all training samples for an individual were processed every two iterations.  We denote the observed data for a mini-batch for subject $s$ as $Y^s_\text{mb}$ and the corresponding latent state as $Z^s_\text{mb}$. In each iteration, we first sampled 
\begin{align*}
    \xi^s_\text{sample} &\sim q_{\phi^s_{\theta, \xi}}(\xi^s) \\
    \mu^s_\text{sample} &\sim q_{\phi^s_{\theta, \mu}}(\mu^s) \\
    \sigma^s_\text{sample} &\sim q_{\phi^s_{\theta, \sigma}}(\sigma^s) \\
    Z^s_{\text{mb}, \text{sample}} &\sim q_{\phi_z^s}(Z^s_\text{mb}),
\end{align*}
using the reparameterization trick \cite{kingma2014}.
Denoting $\theta^s_\text{sample} := \{\xi^s_\text{sample}, \mu^s_\text{sample},  \sigma^s_\text{sample}\}$, we then took a gradient step to optimize the following objective 
\begin{align}
 \sum_{s=1}^S  c^s\log p(Y^s|Z^s_{\text{mb}, \text{sample}}, \theta^s_\text{sample})  - \text{KL}\left[q_{\phi^s_\theta}(\theta^s) || p_\gamma(\theta^s|M^s)\right] -  c^s\text{KL}\left[q_{\phi^s_z}(Z^s_\text{mb}) || p_\lambda(Z^s_\text{mb})\right], 
 \label{eq:sampled_elbo_dr}
\end{align}
where $c^s$ is ratio of the total number of samples in the training data for an individual divided by the number in a mini-batch.  Gradients for both KL terms were computed analytically.

The parameters $\gamma$, $\{\phi_\theta^s\}_{s=1}^S$ and $\{\phi^s_z\}_{s=1}^S$ were initialized with constrained posterior initialization. Constrained posterior initialization was performed as described in \Sect{initialization}, with two minor differences.  First, we tied together the approximate posteriors across individuals for $x^s$ and $\mu^s$ but not for $\sigma^s$.  That is, we enforced $q_{\phi^s_{\theta, \xi}}(\xi^s) = p_{\gamma_\xi}(\xi^s|M^s)$, $q_{\phi^s_{\theta, \mu}}(\mu^s) = p_{\gamma_\mu}(\mu^s|M^s)$ for each $s$ but we optimized $q_{\phi^s_{\theta, \sigma}}(\sigma^s)$ directly for each $s$ without any constraint.  Second, since there were now  approximate posteriors over latent state for each individual, $q_{\phi^s_z}(Z^s)$, these required initialization, and we optimized $\phi^s_z$ directly, without any constraints, during constrained posterior initialization.  The methods of initializing parameters for constrained posterior initialization are listed in \Table{fa_sp_init_details}. During constrained posterior initialization we fixed the parameters of $\gamma$ determining the standard deviation of the CPD over $\xi^s$ and $\mu^s$ so the CPD predicted a constant standard deviation of $.01$ for all properties for these model parameters. Synthesis was performed starting with the initial parameter values produced by constrained posterior initialization. Constrained posterior initialization and synthesis were performed via stochastic gradient ascent with the Adam optimizer \cite{kingma2015}. Gradient ascent for constrained posterior initialization was performed for 1000 epochs with a fixed learning rate of $.01$. Synthesis was performed for 2000 epochs,  starting with a learning rate of $.01$ that was decreased to $.001$ after $500$ epochs.  The decay rates of the moment estimates for the Adam optimizer were fixed at $\beta_1 = .9$ and $\beta_2 = .999$ throughout both constrained posterior initialization and synthesis. Finally, a retrospective form of early stopping was performed by saving checkpoints every 100 epochs and then using the checkpoint with the best model performance on validation data.  See below for details of how validation data was selected.  Model performance for early stopping was based on the average across subjects of the ELBO calculated for validation data.  See below for details on how the ELBO was calculated for validation and test data.

\begin{table}[h]
    \begin{center}
        \begin{tabular}{ |c|c|} 
        \hline
         $\gamma$ & \thead{All entries of $\gamma_{\xi, \mu, j}$ set to 0 for all dimensions $j$\\
         All entries of $\gamma_{\xi, \sigma, j}$ set so that the CPD predicted a constant value for the standard deviation of .01 \\
         for all properties and dimensions $j$ \\
         All entries of $\gamma_{\mu, \mu}$ set to 0\\
         All entries of $\gamma_{\mu, \sigma}$ set so that the CPD predicted a constant value for the standard deviation of .01 \\ for all properties \\
         All entries of $\gamma_{\sigma, \alpha}$ set so that the CPD predicted a constant shape parameter of 10 for all properties \\
         All entries of $\gamma_{\sigma, \beta}$ set so that the CPD predicted a constant rate parameter of 10 for all properties
         }\\
         \hline
         $\phi^s_\theta$ & \thead{ 
         All entries of $\phi^s_{\theta, \sigma, \alpha, i}$ set so shape parameters were $10$ for all neurons $i$ \\
         All entries of $\phi^s_{\theta, \sigma, \beta, i}$ set so rate parameters were $10$ for all neurons $i$
         } \\
         \hline
         $\phi^s_z$ & \thead{All $m^s_i$ set to 0 \\ 
         H initialized as the identity matrix} \\
         \hline 
        \end{tabular}
    \end{center}
    \caption{Methods of initializing parameter values before starting the constrained posterior initialization.  See above for definition of broken out parameters in the right column.  The parameters $\phi^s_{\theta, \xi}$ and $\phi^s_{\theta, \mu}$ are not listed because $q^s_{\phi_{\theta, \xi}}(\xi^s)$ and $q^s_{\phi_{\theta, \mu}}$ are constrained to be equal to the CPD during constrained posterior initialization.}
    \label{tbl:fa_sp_init_details}
\end{table}

\paragraph{Quantifying the ELBO on validation and test data}

We approximated the ELBO on validation and test data for a given subject by optimizing the objective specified in \eq{sampled_elbo_dr} with the following differences. First, since we calculated the ELBO for individual subjects, the sum in \eq{sampled_elbo_dr} was only over a single subject. Second, $Y_\text{mb}^s$ was the entire validation or testing data for an individual, and we set $c^s = 1$. Third, the objective was only optimized with respect to the parameters of the approximate posteriors over latent state and we held the CPD, $p_\gamma(\xi^s, \mu^s, \sigma^s | M^s)$, and approximate posteriors over FA model parameters, $q_{\phi^s_\theta}(\xi^s, \mu^s, \sigma^s)$, fixed.  This was done because we sought to estimate a lower-bound on the log-likelihood of the test data given the posteriors for FA model parameters learned from the training data.  Finally, to improve the accuracy of the approximated ELBO, we used $10$ samples from the approximate posteriors over FA parameters when approximating the expected log-likelihoods for each individual.

\paragraph{Additional details on the analysis of latent-state estimates}

We now provide additional details for the analysis examining latent state estimates obtained with DPMS when we observe only a single behavior in each fish.

We first estimated ground-truth latent-state that would be estimated with standard methods if it is possible observe all three behaviors in each fish.  We did this by fitting standard factor analysis models to all recorded OMR L, R and F conditions separately for fish 8, 9 and 11 from the original dataset \cite{chen2018}.  We used a latent space of 10-dimensions for each model to allow us to compare directly to results obtained with DPMS. Once a model was fit to the data for a fish, latent state was estimated using the mean of the posterior distribution over latent state. 

We then simulated being able to record data for only one condition for each fish by allowing access to only the recorded data for the OMR forward, left and right behaviors from fish 8, 9 and 11, respectively. The observed condition for each fish was selected arbitrarily, and we used all data for the given condition for each fish. We refer to the data for the selected condition for a fish as the ``designated data'' for that fish.  As we detail below, using the designated data for each fish, we then estimated latent state 1) with DPMS, and 2) by separately fitting standard FA models to the designated data for each fish and apply two different methods to put the latent-state estimated from the separate models in the same space. 

When applying DPMS, we 1) first split the designated data for each fish into train and validation data and then synthesized models with this data. We used the training data for model fitting and the validation data for early stopping. After models were synthesized, we 2) then estimated latent state for all data points in the designated data for each fish.   We used the trial-like structure of the original data to split the designated data for each fish into train and validation sets.  In particular, 
the original recordings were made under each of the distinct OMR conditions in multiple discrete periods, and we randomly used two of the periods for the designated condition for a fish as the validation data and the remaining periods of that condition as training data.  To estimate latent state over all data points in the second step, we optimized \eq{sampled_elbo_dr} for each fish with respect to the parameters of the approximate posteriors over latent state, holding the CPD and approximate posteriors over factor analysis model parameters fixed, while setting $Y^s$ to the neural activity for all designated data for a fish.   We then used the means of the posteriors over latent state estimated in this way as point estimates for latent state. 

When applying existing alternative methods, we first used sklearn \cite{scikit-learn} to fit standard FA models to the designated data for each fish independently, again fitting models with 10-dimensional latent spaces, and using the posterior means over latent state to obtain point estimates of latent state for each fish.  We then applied two alternative methods to place these estimates of latent state into the same space:

\begin{enumerate}
    \item{Orthonormalization:
    An important property of factor analysis is that it is possible to find a coordinate system in the latent space such that dimensions in this space map to orthogonal dimensions in the observed space and can be ordered by the amount of variance of observed data they explain.  In principle, ignoring a small set of degenerate models, this provides a means of defining a coordinate system for the latent space of factor analysis models that overcomes the non-identifiabilities inherent to factor analysis and is unique up to the sign of each axis.  A computational efficient ``orthonormalization procedure'' \cite{yu2008} can be applied to find such a latent coordinate system, and combined with a method of addressing the sign degeneracy, this could provide a means of transforming estimates of latent-state across subjects into comparable spaces.  When similar patterns of neural activity explain the same amount of variance across subjects, this method is likely to succeed, but when this assumption is not met, which is again likely when subjects are recorded performing different behaviors, this method may break down.  We demonstrate this by comparing against it.  Specifically, we applied the orthonormalization procedure to the models fit to each subject individually, and resolved the sign degeneracy by arbitrarily selecting one fish as the target fish and selecting the positive direction for each dimension of the latent space for each of the other fish as the direction that minimized the squared error between histograms of latent state projected along that dimension for that that fish and the target fish.}
    \item{Distribution Alignment: If distributions of latent state follow similar distributions across subjects, one approach of mapping latent state estimates across subjects to the same space might be to assume that the distributions of latent state are the same for each subject.  If this assumption is met, methods based
    on optimal transport can be used to find mappings between latent spaces underlying models fit to different subjects. Such approaches have been previously demonstrated with latent variable models fit to real neural data \cite{dyer2017, dabagia2022}.  However, when distributions of latent state are not similar across subjects, as might be the case when subjects display different patterns of neural activity as they perform different behaviors, these approaches may find solutions that incorrectly map latent state estimates from different subjects onto each other when they should in fact be represented in different parts of the latent space.  We emphasize this is not a critique of these approaches. They are designed for use when distributions over latent state are similar across subjects, and they should not be expected to perform well when this assumption is not met. Nonetheless, we feel is valuable to illustrate what can happen when the necessary assumptions are not met to illustrate the general type of problems that can arise and the potential of methods like DPMS, which use additional information such as neural properties, for enabling latent state to still be compared
    across subjects in these scenarios. To illustrate the general problem, we chose to compare against a basic linear form of distribution alignment that finds a linear mapping that maps between the empirical distributions of latent state estimated for two different subjects \cite{Courty2017}, as implemented in the python optimal transport library \cite{flamary2021pot}.  We apply it to three subjects at a time by arbitrarily denoting the distribution of posterior means for one subject as the ``target'' distribution and aligning the distributions over posterior means for the other two subjects to that.}
\end{enumerate}

Finally, we provide details on how we visualized the mappings from latent to neural activity spaces in \Fig{dr_lda}\emph{d}-\emph{f}.  We first sought to identify orthonormal unit vectors, $u_1, \ldots, u_{10}$, in the latent space that explained decreasing amounts of variance in modeled neural activity across the synthesized models for all subjects.  Estimating model parameters with their posterior means and denoting 
posterior mean of $\Lambda_s$ as $\hat{\Lambda}_s$, for any direction $u_i$ in the latent space, the amount of variance in neural activity explained by changes along $u_i$ is proportional to 
\begin{align*}
   \sum_{s=1}^3 ||\hat{\Lambda}_s u_i ||_2^2 =  \left| \left| \begin{bmatrix}\hat{\Lambda}_1 \\ \hat{\Lambda}_2 \\ \hat{\Lambda}_3  \end{bmatrix}u_i \right| \right|_2^2.
\end{align*}
From this it follows that the vectors $u_1, \ldots, u_{10}$ can be found as the right right singular vectors of $\begin{bmatrix}\hat{\Lambda}_1 \\ \hat{\Lambda}_2 \\ \hat{\Lambda}_3  \end{bmatrix}$.  Having identified these directions, we then asked if changes along each of these directions predicted similar changes in observed patterns of neural activity across fish. A change in modeled neural activity for fish $s$ due to a unit-length change along $u_i$ is $\Delta^i_s = \hat{\Lambda}_s u_i \in \mathbb{R}^{d_x^s}$. We visualized these predicted patterns of neural activity by
forming 3-d point clouds of the neuron positions for each fish, associating each point with the predicted change in activity in $\Delta^i_s$ for the corresponding neuron and taking max projections through the resulting volumes. The resulting maps are shown in \Fig{dr_lda}\emph{d}-\emph{f}. 

To illustrate that this consistent mapping between the latent and neural activity spaces was not a trivial consequence of the data, we sought to compare the patterns of neural activity explaining the most variance in the designated data for each fish.  We did this by identifying an orthogonal basis set of patterns explaining decreasing amounts of neural activity according to the FA models fit with standard methods independently to the designated data for each fish.  This basis set can be found through the same orthonormalization procedure previously applied for the purposes of aligning latent spaces.  Once we identified this set of patterns for each fish, we visualized them \Fig{dr_lda}\emph{a}-\emph{c} using the  max-projection procedure just described. 

\paragraph{Additional details on the quantification of model performance across behaviors}

We now provide additional details on how we quantified the ability of DPMS to synthesize factor analysis models with structure accounting for different behaviors observed in different animals. For this quantitative analysis, in addition to train and validation data, we also required test data.  We first broke the periods of neural data recorded under any of the three OMR behavioral conditions (forward, right, left) into chunks of 5 time points each for each fish. This was done so that when chunks were assigned to train, validation and test sets, the temporal dependence between training and validation and test data was reduced relative to assigning individual time points.
For each target fish, we used a paired six-fold cross-validation design. For each behavioral condition, equal-size disjoint sets of target-fish chunks were assigned to six test folds. The same held-out target-fish test chunks were used for the SB and DB scenarios. Under both the SB and DB scenarios the training data for a target fish was the same. However, under the SB scenario the training and validation data for the non-target fish represented the same behavior as the target fish, while the training and validation data represented the two different behaviors (one behavior for each non-target fish) under the DB scenario.  In all cases, the total amount of training and validation data was matched across SB and DB scenarios. Model fitting was performed as described above: models were synthesized using training data, with validation data used for retrospective early stopping. Early stopping was based only on the target-fish validation ELBO on held-out target-fish validation data, as described above.
We analyzed diagonal and off-diagonal train-test behavior pairs, corresponding to matched and mismatched train-test condition pairings respectively, separately. Condition comparisons were combined across all fish and folds, yielding 54 diagonal and 108 off-diagonal comparisons. We then applied exact two-sided sign-count tests to assess whether DB outperformed SB, with Holm correction across the two pair types.


\subsection{Constrained posterior initialization}
\label{sec:initialization}

We motivate and describe an important way of initializing the approximate posteriors and CPD when synthesizing models with DPMS. This initialization procedure can be particularly important when working with simplified, unimodal distributions for the approximate posteriors and CPD. 

A simple example of when careful initialization is required is when the magnitude but not the sign of certain parameters for the models for each system instance matter (formally, the sign is non-identifiable).  In this case, unimodal approximate posteriors, unable to represent the required bimodal distributions, may place most of their probability density around values of the correct magnitude but only a single and different sign across system instances. From \Thm{main}, the optimal CPD would be be the average of these posteriors.  However, if the form of the CPD is such that conditioned on one set of properties it too can only represent unimodal distributions, it will be unable to represent the optimal bimodal solution - instead likely taking on a form with a mean of zero and large variance.  This is not a very meaningful CPD.  An arguably better solution, given the constrained distributional forms in use, would be a solution where the approximate posteriors concentrate around values of the same sign for each example system.  A unimodal CPD could be better fit to the average of such posteriors, and this would be essentially equivalent to doing nothing more than having selected a convention up front to break the non-identifiability inherent in the model forms for each system instance being fit.  Similar concerns can arise when synthesizing models with other forms of non-identifiability in their parameterizations. 

We now describe how we can arrive at approximate posteriors that handle non-identifiabilities in similar ways.  One approach to this is to initialize the approximate posteriors in a manner so that similar solutions will be favored when optimizing the ELBO with gradient ascent.  For example, in the above example, we might initialize approximate posteriors to favor parameters of the same sign.  However, this general strategy will often be difficult in practice because the non-identifiabilities in complex models may not be immediately apparent nor may it be clear how they should be initialized to achieve the desired ends. An alternative solution, which we pursue in this work and refer to as constrained posterior initialization, is to learn good initializations for the approximate posteriors that handle non-identifiabilities in the same manner. Using the notation of \Sect{practical_concerns}, we do this by performing an initialization step where we solve a constrained synthesis problem where we tie the approximate posteriors for $\theta_\text{props}$ across system instances together by constraining them to be equal to the CPD.  This initialization step may provide reasonable starting points for explaining the data for each system instance and will, by design, produce initial approximate posteriors that handle non-identifiabilities in the same manner.  We perform this initialization by enforcing 
\begin{align*}
    q_{\phi^s_\text{props}}(\theta^s_\text{props}) = p_\gamma(\theta^s_\text{props}|M^s)
\end{align*}
for each system instance and then optimizing the ELBO as we would normally. Practically, this can be implemented by 1) using $p_\gamma(\theta^s_\text{props}|M^s)$ in place of $q_{\phi^s_\text{props}}(\theta^s_\text{props})$ in \Eq{ELBO_FULL}, 2) optimizing this modified objective as normal to learn initial values, $\gamma_\text{init}$, for the CPD, and then 3) setting $q_{\phi^s_\text{props}}(\theta^s_\text{props}) = p_{\gamma_\text{init}}(\theta^s_\text{props}|M^s)$ for each $s$. When performing this constrained synthesis problem, we optimize any $\theta_\text{shared}$, $\delta$ and $\phi^s_\text{no-props}$ parameters as normal. After doing this, synthesis is performed as normal using the set of starting with values of all parameters optimized in the initialization step.  As a final practical matter, we find it helpful to fix the parameters of the CPD determining variance during the initialization step so that estimates of variance do not become small.  When the CPD represents conditional Gaussians, this is achieved by just optimizing the parameters of the mean function, while leaving those for the variance function fixed.  We believe in practice this can be helpful because initializing with approximate posteriors with larger variances when performing DPMS allows larger regions of the parameter space to be initially explored when we are sampling from the approximate posteriors to approximate the ELBO.

\newpage
\section{Supplemental Figures}
\setcounter{figure}{0}

\begin{figure}[h]
\centering
\includegraphics[width=5.5in]{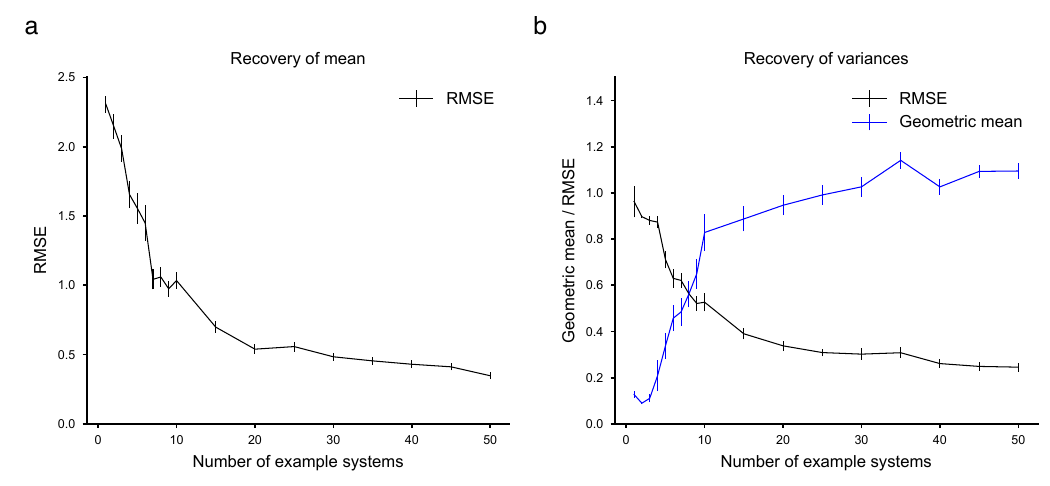}
\caption{\emph{Accuracy of the learned CPD for linear systems when only one sample is observed from each example system. } We examine the ability of DPMS to learn the CPD for linear systems when only one sample is observed from each example system.
To examine the impact of varying the number of system instances data is collected from, we perform a number of simulations, varying the number of system instances in each.  For simplicity, we assume all system instances have the same properties, so the CPD simplifies to a non-conditional prior.  This can also be understood as examining how the CPD is learned locally for one value of conditioning properties, so results here should be indicative of performance for the more general case. In each simulation, system instance $s$ generates an output $y^s \in \mathbb{R}$ given input, $x^s \in \mathbb{R}^5 \sim \mathcal{N}(0, I)$ according to $y^s = (\theta^s)^T x^s + r^s$, where $r^s \sim \mathcal{N}(0, 1)$.  Since only one sample is observed from each example system, we do not include a sample number index for $x^s, y^s$ or $r^s$.  Weights, $\theta^s$, for each system instance are generated from a $\mathcal{N}\left(\mu, I \right)$ prior, which represents the ground-truth CPD, where $\mu = [0, 1, 2, 3, 4]^T$.  We then apply DPMS to learn the CPD, which we assume is of the form $\mathcal{N}(\hat{\mu}, \hat{\Sigma})$, where we seek to learn $\hat{\mu}$ and $\hat{\Sigma}$, and where $\hat{\Sigma}$ is constrained to be diagonal.  We employ a general multivariate normal distribution (of the same form as $q_{\phi_z^s}(Z^s)$ specified in \sect{add_fa_methods}) for the approximate posterior, $q(\theta^s)$, for each example system, and for simplicity, we assume the variance of $r^s$ is known. When fitting, we perform 1500 gradient ascent iterations, which was long enough to ensure convergence of the ELBO in all simulations. To examine how performance varies with the number of system instances data is collected from, we vary the number of system instances in a simulation from 1 to 50.  We perform 30 independent simulations for each number of system instances.  For each simulation, we measure the accuracy of the learned CPD using root mean square error (RMSE) between the learned $\hat{\mu}$ and true $\mu$ and RMSE between the diagonal of the learned $\hat{\Sigma}$ and true covariance $I$.  We also calculate the geometric mean of the diagonal of $\hat{\Sigma}$.  The geometric mean is directly related to the determinant of $\hat{\Sigma}$, which is one means of quantifying the amount of uncertainty in the CPD. The geometric mean of the covariance matrix, $I$, for the true CPD is 1. (\emph{\textbf{a}}) RMSE between the true and learned mean of the CPD as the number of system instances samples are observed from increases.  Values and error bars indicate the mean and standard error across simulations. (\emph{\textbf{b}}) RMSE (black) between the diagonal of the learned and true covariance of the CPD and geometric mean (blue) of the learned covariance as the number of system instances samples are observed from increases.  }
\label{fig:low_data_synth}
\end{figure}

\begin{figure}[h]
\centering
\includegraphics[width=5.5in]{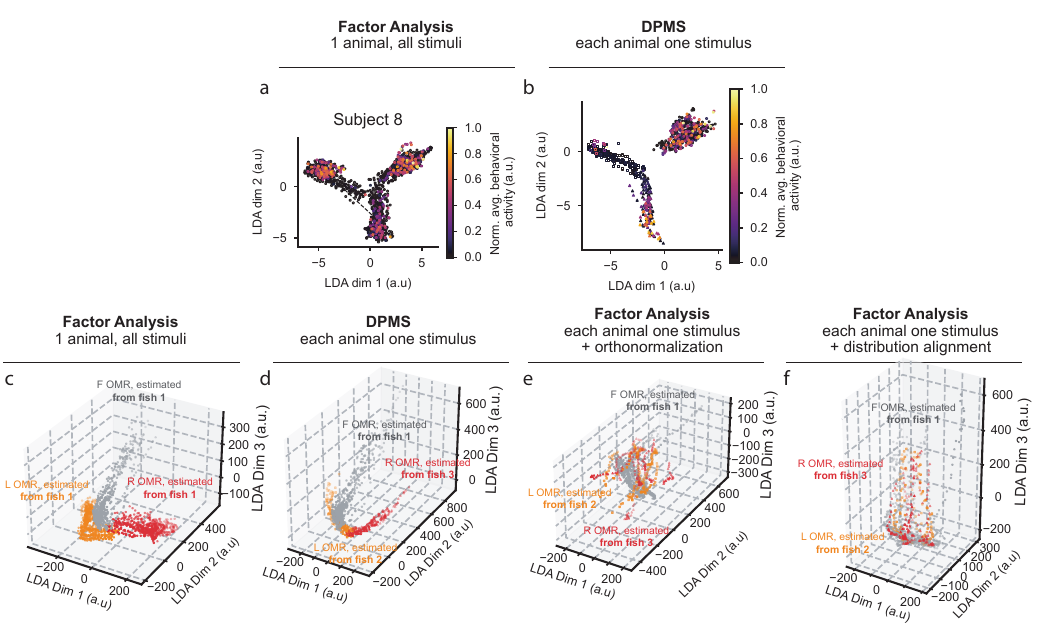}
\caption{\emph{Additional panels on dimensionality reduction} 
(\emph{\textbf{a}, \textbf{b}}) Plotting the average neural activity (normalized between 0-1) onto the latent states shows that moments of no swimming lie near the center of the latent space, whereas moments of swimming are away from the center for each cluster. (\emph{\textbf{c}}-\emph{\textbf{f}}) From left to right, in similar order as \Fig{low_data_synth}\emph{\textbf{e}}-\emph{\textbf{h}}, latent state estimated with model synthesis, applying orthonormalization to latent state estimated across fish with standard factor analysis, and applying DA to latent state estimated across fish with standard factor analysis, each when data from only a single behavior is observed in each fish. Latent state is shown in the best three-dimensional spaces for differentiating behavior for each approach, and coordinate axes have been reflected and rotated to visually correspond to those in panel \emph{a}.}
\label{fig:low_data_synth_supp}
\end{figure}

\begin{figure}[h]
\begin{center}
\includegraphics[width=4in]{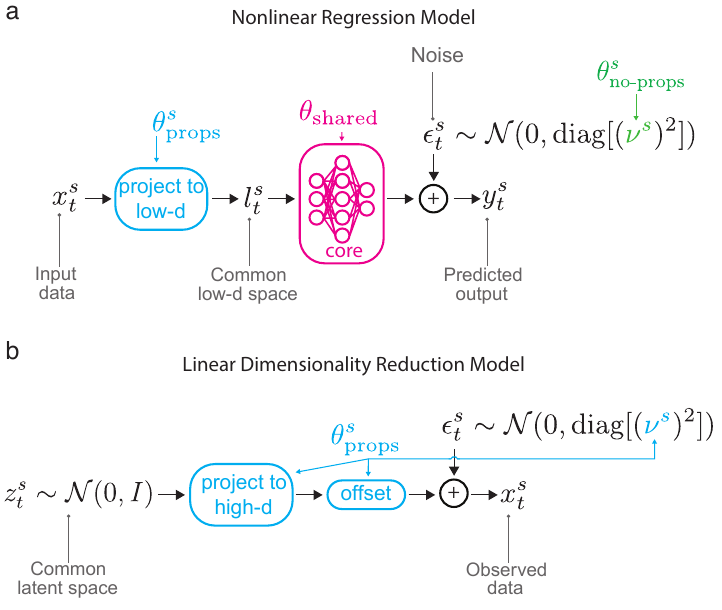}
\caption{\emph{Example model forms used in this work.} (\emph{\textbf{a}}) A model for non-linear regression, where high-dimensional input data is projected into a common low-dimensional space, shared across system instances. Once data is projected to the common space, a ``core'' module predicts the output of the system. The core module is shared across system instances, and so captures the interesting, non-linear behavior of the system under study in a way that generalizes across systems instances. The goal of model synthesis is to learn to predict the projection weights to the common space for each system instance from measurable properties. (\emph{\textbf{b}}) A model for linear dimensionality reduction. Latent variables reside in a common space that is shared across system instances, enabling latent variables for different system instances to be examined in the same space. The goal of model synthesis is to predict the coefficients of the linear mapping between latent variables to the observed activity of each system instance as well as the standard deviation of the noise for each observed variable from measurable properties. In both panels, colors correspond to the type of parameters for the different parts of each model. As illustrated, models do not necessarily need to incorporate $\theta_\text{shared}$ or $\theta^s_\text{no-props}$ parameters.}
\label{fig:models}
\end{center}
\end{figure}

`


\end{document}